\documentclass{article} 
\usepackage{iclr2027_conference,times}
\iclrfinalcopy

\usepackage{hyperref}
\usepackage{url}
\usepackage{amsmath,amssymb,amsthm,mathtools}
\usepackage{booktabs}
\usepackage{algorithm}
\usepackage{algorithmic}
\usepackage{xcolor}
\usepackage{microtype}
\usepackage{array}
\usepackage{subcaption}
\usepackage{graphicx}
\usepackage{tikz}
\usepackage{enumitem}
\usepackage{wrapfig}
\usetikzlibrary{arrows.meta,positioning,calc}

\theoremstyle{plain}
\newtheorem{theorem}{Theorem}[section]
\newtheorem{lemma}[theorem]{Lemma}

\theoremstyle{definition}
\newtheorem{definition}[theorem]{Definition}

\theoremstyle{remark}
\newtheorem{remark}[theorem]{Remark}

\newcommand{\R}{\mathbb{R}}
\newcommand{\Sn}[1]{S_{#1}}
\newcommand{\E}{\mathbf{E}}      
\newcommand{\Hs}{\mathbf{H}}     
\newcommand{\G}{\mathbf{G}}      
\newcommand{\dE}{d_{E}}
\newcommand{\dH}{d_{H}}
\newcommand{\Ifs}{\mathcal{I}}   
\newcommand{\Ops}{\mathcal{O}}   
\newcommand{\Gam}{\Gamma}

\newcommand{\norm}{\operatorname{norm}}
\DeclareMathOperator{\cat}{cat}

\DeclareMathOperator{\softmax}{softmax}
\DeclareMathOperator{\gelu}{GELU}
\DeclareMathOperator{\mean}{mean}

\title{Interfaces and Modules: A Scalable Representation for Weight-Space Learning across Architectures}

\author{Micha\l\ Kozyra \\
University of Oxford \\
\texttt{michal.kozyra@seh.ox.ac.uk}
\And
Michael Bronstein \\
University of Oxford\\
AITHYRA
}

\begin{document}

\maketitle
\lhead{}
\renewcommand{\headrulewidth}{0pt}

\begin{abstract}
Weight-space learning across architectures requires representations
that respect parameter symmetries and scale to large networks.
Existing equivariant tensor encoders require architecture-specific
symmetry specifications, while neuron graphs make global attention
costly. We introduce a typed graph organized around shared
coordinate-permutation actions. Its nodes, called \emph{interfaces},
represent coordinate spaces whose bound tensor axes must be
permuted jointly; operation occurrences, called \emph{modules},
form typed hyperedges. A compiler derives this graph from supported PyTorch models,
with exact recovery of the permutation factors defined by
its constraints.
The Universal Weight Encoder (UWE) combines local equivariant
updates to parameter features with global attention over invariant
interface summaries. The small number of interfaces makes global
attention practical, and a single encoder processes different
supported architectures. UWE is competitive with
architecture-specific state-of-the-art methods on established
benchmarks. On a new zoo of 1,200 transformers spanning
100 architectures, it outperforms the evaluated baselines at
ranking networks within unseen architectures, including
architectures larger than those used for training, while
using substantially less memory.
\end{abstract}

\section{Introduction}
\label{sec:intro}

Weight-space learning treats trained neural networks as data:
predicting performance from parameters, editing implicit neural
representations, and generating weights and adapters
\citep{unterthiner2020,navon2023,zhou2023nfn,
mahabadi2021hyperformer,ruiz2024hyperdreambooth}.
Public model hubs motivate extending these methods across networks
that differ in width, depth, operations and training history
\citep{falk2025}.
Their representations must also account for parameter symmetries:
hidden-unit permutations, and in suitable architectures continuous
rescalings and changes of basis, can preserve the represented
function
\citep{hechtnielsen1990,chen1993,neyshabur2015pathsgd,
putterman2024lora}.
Here we focus on coordinate permutations.
We seek a representation that respects these symmetries, supports
modern architectural components, transfers across architectures,
and remains compact enough for global attention.
It should also be derived from the network's implementation.

Existing representations meet only some of these requirements.
Equivariant weight-space networks act directly on weight tensors,
using an architecture-specific construction or a supplied
permutation specification
\citep{navon2023,zhou2023nfn,zhou2024unf,tran2025tnfn}.
Neural graph methods instead represent neurons as nodes and
weights as edge features, allowing one encoder to process different
architectures \citep{kofinas2024,lim2024}.
However, extending these representations to combinations of
residual connections, normalization, attention and low-rank adapters
requires careful treatment of their parameter dependencies.
The required symmetry specifications or graph encodings are
typically supplied by hand, and their correctness is not automatic:
the attention representation of \citet{kofinas2024} admits
relabellings that change the function and excludes valid ones
(Section~\ref{sec:attention-example}).

The size of neural graphs creates a further constraint.
Local message passing requires successive hops to connect distant
operations, while fully connected attention is quadratic in the
number of neuron nodes.
Methods that tokenize windows of weights offer another route to
scale, but do not enforce permutation equivariance
\citep{schurholt2024,falk2025}.
We therefore need objects that preserve the shared permutation
actions of the network while being substantially fewer than its
neurons.
Appendix~\ref{app:related} reviews related work in more detail.

These requirements suggest separating operations from the
coordinate spaces they act on.
In an MLP, permuting a hidden layer requires the same relabelling
of the preceding weight matrix's outputs and the following
matrix's inputs.
A residual addition can impose this agreement across several
activation tensors.
The shared coordinate space is therefore a natural object for
the representation: it identifies which parameter axes must
be relabelled together.
We call such a space an \emph{interface}.
Each permutation-bearing interface carries one shared coordinate
permutation, while distinct interfaces may be permuted independently.
We represent the network as a typed graph whose nodes are interfaces
and whose hyperedges are operation occurrences, called
\emph{modules}.

Our main contributions are this \emph{interface--module graph}
and the \emph{Universal Weight Encoder} (UWE), an equivariant
encoder that operates on it across architectures.
Shared permutation bindings already appear in alignment
specifications \citep{ainsworth2023} and in specifications used
to construct equivariant layers \citep{zhou2024unf}.
We use these bindings to define the graph's nodes and the
objects over which the encoder communicates.
UWE alternates local equivariant updates to parameter features
with global attention over invariant interface summaries
(Section~\ref{sec:encoder}).
The graphs are compact: a ResNet-18 has twelve
permutation-bearing interfaces, and an 85M-parameter transformer
has 301, compared with about 84,000 neuron nodes.
Global attention over these interfaces is therefore affordable.

To automate graph construction, a compiler traces a supported
PyTorch model and infers interfaces and parameter-axis bindings
from operator constraints.
Recovery is exact within its constraint language, and the
recovered relabellings preserve the network's function
(Section~\ref{sec:compiler}).
Models assembled from supported operations, layouts and sharing
patterns require no new architecture-specific specification.
Figure~\ref{fig:uwe-pipeline} illustrates the pipeline.

\begin{figure}[t]
\centering
\includegraphics[width=\linewidth]{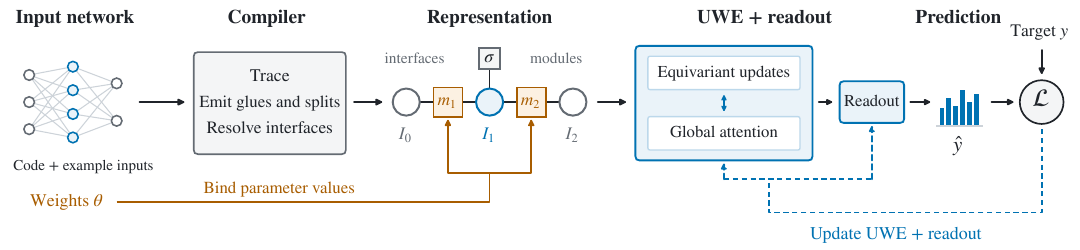}
\caption{The compiler derives interfaces and parameter bindings
from network code. UWE processes the resulting graph, and a
supervised loss trains the encoder and readout.}
\label{fig:uwe-pipeline}
\end{figure}

Evaluating transfer across architectures requires benchmarks
that vary both architecture and scale.
Common weight-space benchmarks focus on small networks from
a single architecture or a narrow family, leaving transfer to
larger, unseen architectures poorly tested.
We therefore introduce a transformer zoo with architecture-level
splits and loss profiles across class groups and image corruptions.
These profiles distinguish networks with similar aggregate
performance but different patterns of error.
UWE is competitive on standard weight-space benchmarks and
outperforms the tested baselines on the new zoo, with stronger
transfer to larger held-out architectures and lower memory use
(Section~\ref{sec:results}).
Separate scaling experiments establish computational feasibility
on larger networks, including full training steps on ViT-L/16
and GPT-2 Medium on a single 80\,GB GPU
(Section~\ref{sec:scaling-main}).

\section{The representation: a typed graph of interfaces and modules}
\label{sec:representation}

\subsection{Objects}
\label{sec:objects}

The representation is a typed bipartite graph with three kinds of
object (Figure~\ref{fig:overview}).

An \emph{interface} $I\in\Ifs$ is a coordinate space of width $n_I$.
It is \emph{permutation-bearing} if the compiler's constraints permit
arbitrary coordinate permutations, and \emph{fixed} if they permit
only the identity. Fixed interfaces arise, for example, from
input/output layers or fixed positional encodings. Interfaces carry
no arbitrary identity or depth labels.

A \emph{module} $o\in\Ops$ is an operation occurrence with a type,
hyperparameter descriptor and network depth. Several modules may
reference shared parameters. Parameter-free operations, such as
activations, additions and concatenations, are \emph{weightless}
modules retained for their constraints on relabellings.

A \emph{binding} associates a module axis with an interface under
a \emph{role}, such as input, output, query-output or scale-and-shift,
from a fixed vocabulary of about a dozen entries
(Appendix~\ref{app:vocab}). Each bound axis has one binding;
several axes may bind to the same interface. The representation
is thus a typed hypergraph: modules are typed hyperedges, and
bindings are their role-labelled incidences.

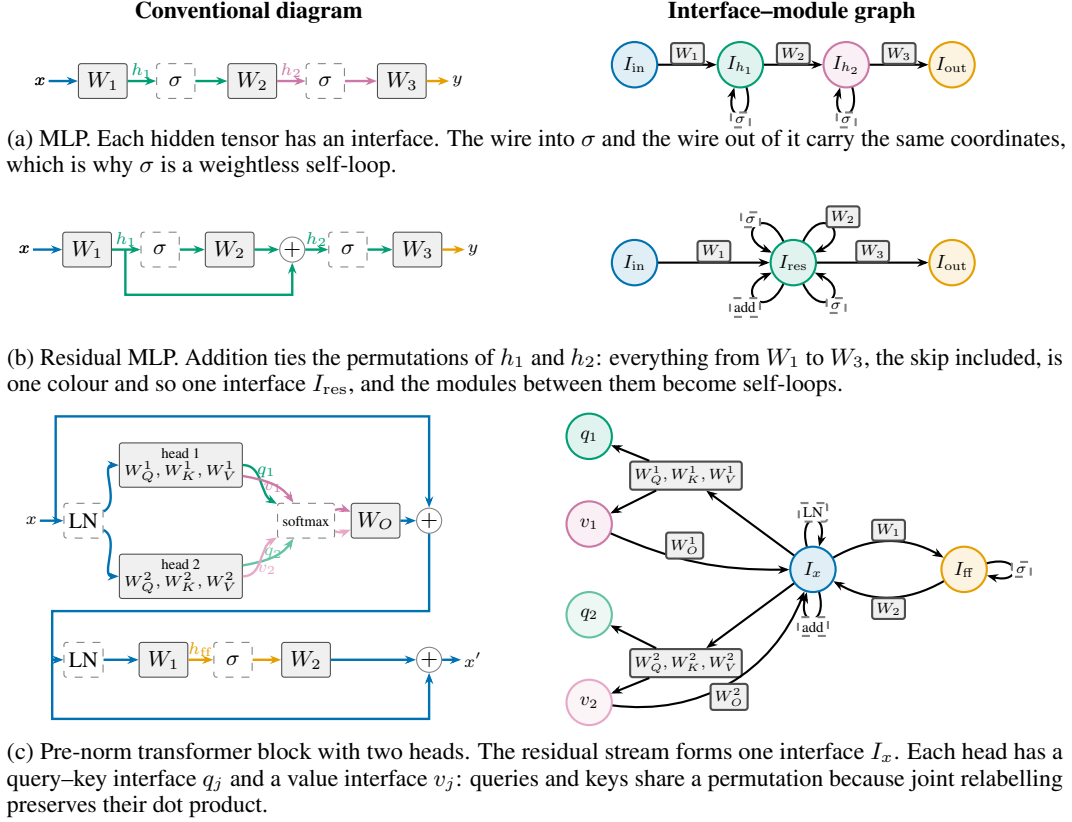
\begin{figure}[t]
\centering
\definecolor{ifA}{HTML}{0072B2} 
\definecolor{ifB}{HTML}{009E73} 
\definecolor{ifC}{HTML}{CC79A7} 
\definecolor{ifD}{HTML}{E69F00} 
\definecolor{ifB2}{HTML}{66C2A5} 
\definecolor{ifC2}{HTML}{E7A8C8} 
\definecolor{modgrey}{HTML}{4D4D4D}
\tikzset{
  iface/.style={circle, draw=#1, fill=#1!12, thick, minimum size=7.5mm, font=\small, inner sep=1pt},
  iface/.default=ifA,
  mod/.style={font=\scriptsize, inner sep=2pt, fill=gray!12, draw=modgrey, rounded corners=1pt},
  wl/.style={font=\scriptsize, inner sep=2pt, fill=white, draw=modgrey!70, dashed, rounded corners=1pt},
  rep/.style={>={Stealth[length=1.5mm]}, thick},
  op/.style={draw=modgrey, fill=gray!12, rounded corners=1pt, minimum height=5mm, minimum width=7mm, inner sep=2pt, font=\small},
  wlop/.style={draw=modgrey!70, dashed, fill=white, rounded corners=1pt, minimum height=5mm, minimum width=5.5mm, inner sep=2pt, font=\small},
  sum/.style={circle, draw=modgrey!70, fill=white, inner sep=0pt, minimum size=3.6mm, font=\small},
  ten/.style={font=\scriptsize, inner sep=1pt},
  w/.style={->, line width=0.9pt, draw=#1, >={Stealth[length=1.5mm]}},
  w/.default=ifA,
  lab/.style={font=\scriptsize, inner sep=1pt, text=#1}}
\begin{tabular}{@{}c@{\hspace{4mm}}c@{}}
\textbf{\small Conventional diagram} & \textbf{\small Interface--module graph} \\[3pt]
\begin{tikzpicture}[baseline=(current bounding box.center), scale=0.92, every node/.style={transform shape}]
\node[ten] (x) at (0,0) {$x$};
\node[ten] (x) at (0,0) {$x$};
\node[op]  (w1) at (0.9,0) {$W_1$};
\node[wlop] (s1) at (1.95,0) {$\sigma$};
\node[op]  (w2) at (3.05,0) {$W_2$};
\node[wlop] (s2) at (4.1,0) {$\sigma$};
\node[op]  (w3) at (5.2,0) {$W_3$};
\node[ten] (y) at (6.0,0) {$y$};
\draw[w=ifA] (x) -- (w1);
\draw[w=ifB] (w1) -- node[lab=ifB, above] {$h_1$} (s1);
\draw[w=ifB] (s1) -- (w2);
\draw[w=ifC] (w2) -- node[lab=ifC, above] {$h_2$} (s2);
\draw[w=ifC] (s2) -- (w3);
\draw[w=ifD] (w3) -- (y);
\end{tikzpicture}
&
\begin{tikzpicture}[rep, scale=0.8, every node/.style={transform shape}, baseline=(current bounding box.center)]
\useasboundingbox (-0.45,-1.0) rectangle (5.65,0.5);
\node[iface=ifA] (in)  at (0,0) {$I_{\mathrm{in}}$};
\node[iface=ifB] (h1)  at (1.75,0) {$I_{h_1}$};
\node[iface=ifC] (h2)  at (3.5,0) {$I_{h_2}$};
\node[iface=ifD] (out) at (5.25,0) {$I_{\mathrm{out}}$};
\draw[->] (in) -- node[mod, above] {$W_1$} (h1);
\draw[->] (h1) -- node[mod, above] {$W_2$} (h2);
\draw[->] (h2) -- node[mod, above] {$W_3$} (out);
\path (h1) edge[loop below, min distance=6mm, looseness=6, ->] node[wl, below, yshift=1pt] {$\sigma$} (h1);
\path (h2) edge[loop below, min distance=6mm, looseness=6, ->] node[wl, below, yshift=1pt] {$\sigma$} (h2);
\end{tikzpicture}
\\[1mm]
\multicolumn{2}{@{}p{\textwidth}@{}}{\small (a) MLP. Each hidden tensor has an interface.
The wire into $\sigma$ and the wire out of it carry the same coordinates, which
is why $\sigma$ is a weightless self-loop.}\\[3mm]
\begin{tikzpicture}[baseline=(current bounding box.center), scale=0.92, every node/.style={transform shape}]
\node[ten] (x) at (0,0) {$x$};
\node[ten] (x) at (0,0) {$x$};
\node[op]  (w1) at (0.9,0) {$W_1$};
\node[wlop] (s1) at (1.95,0) {$\sigma$};
\node[op]  (w2) at (2.95,0) {$W_2$};
\node[sum] (add) at (3.85,0) {$+$};
\node[wlop] (s2) at (4.65,0) {$\sigma$};
\node[op]  (w3) at (5.65,0) {$W_3$};
\node[ten] (y) at (6.45,0) {$y$};
\draw[w=ifA] (x) -- (w1);
\draw[w=ifB] (w1) -- node[lab=ifB, above] {$h_1$} (s1);
\draw[w=ifB] (s1) -- (w2);
\draw[w=ifB] (w2) -- (add);
\draw[w=ifB] (add) -- node[lab=ifB, above] {$h_2$} (s2);
\draw[w=ifB] (s2) -- (w3);
\draw[w=ifD] (w3) -- (y);
\draw[w=ifB] (1.45,0) |- (2.6,-0.65) -| (add);
\end{tikzpicture}
&
\begin{tikzpicture}[rep, scale=0.8, every node/.style={transform shape}, baseline=(current bounding box.center)]
\useasboundingbox (-0.45,-1.35) rectangle (5.65,1.35);
\node[iface=ifA] (in)  at (0,0) {$I_{\mathrm{in}}$};
\node[iface=ifB] (res) at (2.625,0) {$I_{\mathrm{res}}$};
\node[iface=ifD] (out) at (5.25,0) {$I_{\mathrm{out}}$};
\draw[->] (in)  -- node[mod, above] {$W_1$} (res);
\draw[->] (res) -- node[mod, above] {$W_3$} (out);
\path (res) edge[loop, out=115, in=155, min distance=7mm, looseness=6, ->] node[wl, above left, xshift=2pt, yshift=-1pt] {$\sigma$} (res);
\path (res) edge[loop, out=65,  in=25,  min distance=7mm, looseness=6, ->] node[mod, above right, xshift=-2pt, yshift=-1pt] {$W_2$} (res);
\path (res) edge[loop, out=245, in=205, min distance=7mm, looseness=6, ->] node[wl, below left, xshift=2pt, yshift=1pt] {add} (res);
\path (res) edge[loop, out=295, in=335, min distance=7mm, looseness=6, ->] node[wl, below right, xshift=-2pt, yshift=1pt] {$\sigma$} (res);
\end{tikzpicture}
\\[1mm]
\multicolumn{2}{@{}p{\textwidth}@{}}{\small (b) Residual MLP. Addition ties the permutations of $h_1$
and $h_2$: everything from $W_1$ to $W_3$, the skip included, is one colour and
so one interface $I_{\mathrm{res}}$, and the modules between them become
self-loops.}\\[3mm]
\begin{tikzpicture}[baseline=(current bounding box.center), scale=0.92, every node/.style={transform shape}]
\useasboundingbox (-0.35,-2.1) rectangle (6.6,2.7);
\node[ten] (x) at (0,1.0) {$x$};
\node[wlop] (ln1) at (0.75,1.0) {LN};
\node[op, align=center, font=\tiny, minimum width=15mm] (h1) at (2.15,1.8)
      {head 1\\$W_Q^{1},W_K^{1},W_V^{1}$};
\node[op, align=center, font=\tiny, minimum width=15mm] (h2) at (2.15,0.2)
      {head 2\\$W_Q^{2},W_K^{2},W_V^{2}$};
\node[wlop, font=\tiny] (att) at (3.95,1.0) {softmax};
\node[op] (wo) at (4.95,1.0) {$W_O$};
\node[sum] (a1) at (5.72,1.0) {$+$};
\draw[w=ifA] (x) -- (ln1);
\draw[w=ifA] (ln1) to[out=25, in=180] (h1);
\draw[w=ifA] (ln1) to[out=-25, in=180] (h2);
\draw[w=ifB] (h1) to[out=0, in=150] node[lab=ifB, above right, xshift=-3pt] {$q_1$} (att);
\draw[w=ifC] (h1) to[out=-10, in=125] node[lab=ifC, below left, xshift=6pt, yshift=2pt] {$v_1$} (att);
\draw[w=ifB2] (h2) to[out=10, in=235] node[lab=ifB2, above left, xshift=6pt, yshift=-2pt] {$q_2$} (att);
\draw[w=ifC2] (h2) to[out=0, in=210] node[lab=ifC2, below right, xshift=-3pt] {$v_2$} (att);
\draw[w=ifC]  (att) to[out=20, in=160] (wo);
\draw[w=ifC2] (att) to[out=-20, in=200] (wo);
\draw[w=ifA] (wo) -- (a1);
\draw[w=ifA] (0.35,1.0) |- (2.6,2.5) -| (a1);
\node[wlop] (ln2) at (0.75,-1.0) {LN};
\node[op]  (w1) at (1.9,-1.0) {$W_1$};
\node[wlop] (s) at (2.9,-1.0) {$\sigma$};
\node[op]  (w2) at (3.95,-1.0) {$W_2$};
\node[sum] (a2) at (5.72,-1.0) {$+$};
\node[ten] (y) at (6.35,-1.0) {$x'$};
\draw[w=ifA] (a1) -- (5.72,-0.3) -- (0.3,-0.3) -- (0.3,-1.0) -- (ln2);
\draw[w=ifA] (ln2) -- (w1);
\draw[w=ifD] (w1) -- node[lab=ifD, above] {$h_{\mathrm{ff}}$} (s);
\draw[w=ifD] (s) -- (w2);
\draw[w=ifA] (w2) -- (a2);
\draw[w=ifA] (a2) -- (y);
\draw[w=ifA] (0.3,-1.0) -- (0.3,-1.85) -| (a2);
\end{tikzpicture}
&
\begin{tikzpicture}[rep, scale=0.8, every node/.style={transform shape}, baseline=(current bounding box.center)]
\useasboundingbox (-4.3,-2.8) rectangle (3.6,2.8);
\node[iface=ifA] (x) at (0,0) {$I_{x}$};
\node[mod] (hd1) at (-2.1,1.55)  {$W_Q^{1},W_K^{1},W_V^{1}$};
\node[iface=ifB] (q1) at (-3.7,2.2)  {$q_1$};
\node[iface=ifC] (v1) at (-3.7,0.75)  {$v_1$};
\node[mod] (hd2) at (-2.1,-1.55) {$W_Q^{2},W_K^{2},W_V^{2}$};
\node[iface=ifB2] (q2) at (-3.7,-0.75) {$q_2$};
\node[iface=ifC2] (v2) at (-3.7,-2.2) {$v_2$};
\draw[->] (x) -- (hd1); \draw[->] (hd1) -- (q1); \draw[->] (hd1) -- (v1);
\draw[->] (x) -- (hd2); \draw[->] (hd2) -- (q2); \draw[->] (hd2) -- (v2);
\draw[->] (v1) to[bend right=12] node[mod, above, pos=0.42] {$W_O^{1}$} (x);
\draw[->] (v2) to[bend right=38] node[mod, below, pos=0.5] {$W_O^{2}$} (x);
\node[iface=ifD] (ff) at (2.5,0) {$I_{\mathrm{ff}}$};
\draw[->] (x)  to[bend left=30] node[mod, above] {$W_1$} (ff);
\draw[->] (ff) to[bend left=30] node[mod, below] {$W_2$} (x);
\path (ff) edge[loop right, min distance=6mm, looseness=6, ->] node[wl, right, xshift=-1pt] {$\sigma$} (ff);
\path (x)  edge[loop above, min distance=6mm, looseness=6, ->] node[wl, above, yshift=-1pt] {LN}  (x);
\path (x)  edge[loop below, min distance=6mm, looseness=6, ->] node[wl, below, yshift=1pt] {add} (x);
\end{tikzpicture}
\\[1mm]
\multicolumn{2}{@{}p{\textwidth}@{}}{\small
(c) Pre-norm transformer block with two heads. The residual stream forms
one interface $I_x$. Each head has a query--key interface $q_j$ and a
value interface $v_j$: queries and keys share a permutation because
joint relabelling preserves their dot product.}
\end{tabular}
\caption{Conventional diagrams (left) and interface--module graphs
(right). Matching colours identify corresponding coordinate spaces;
shades distinguish heads. Objects sharing a coordinate-permutation
action map to one interface. Operations whose inputs and outputs share
an interface become self-loops. The compiler
(Section~\ref{sec:compiler}) constructs these graphs for the encoder
(Section~\ref{sec:encoder}).}
\label{fig:overview}
\end{figure}

\subsection{Parameters}
\label{sec:units}

A parameter tensor with two channel axes bound to $I$ and $J$ forms a
\emph{rectangular} block of cells $W\in\R^{n_I\times n_J\times k}$;
a bias or normalization vector with one bound axis forms a
\emph{diagonal} block $W\in\R^{n_I\times k}$. The trailing axis
folds $k$ scalars into each cell, such as kernel elements or
normalization scale and shift; its coordinates are not permuted.
A block with its bound axes is a \emph{unit} $u$.
Modules reference one or more units, such as a linear layer's
weight and bias; several occurrences may share a unit. An axis spanning
several interfaces, as when reading a concatenation, is partitioned
into one unit per interface. The representation re-indexes parameters;
inverting the layout recovers the original tensors. Each distinct parameter tensor is represented once.

\subsection{Asserted symmetry}
\label{sec:asserts}

A \emph{relabelling} permutes hidden coordinates and transforms
parameters through their bound axes. For a weight matrix
$W\in\R^{n_J\times n_I}$ from $I$ to $J$ and its bias,
\begin{equation}
  W \longmapsto P_J\, W\, P_I^{\top},\qquad b \longmapsto P_J\, b.
  \label{eq:action}
\end{equation}
where $P_I$ and $P_J$ are the corresponding permutation matrices.

\begin{definition}[Asserted symmetry]
\label{def:asserted}
Independently choose $\pi_I\in\Sn{n_I}$ for each permutation-bearing
interface and the identity for each fixed interface. Apply each
permutation to every stored-state axis bound to that interface,
respecting occurrence-level bindings and alias constraints.
These transformations form
\[
  \Gam=\prod_{I\in\Ifs^\ast}\Sn{n_I}.
\]
Every element of $\Gam$ is asserted to preserve the network's function.
The group permutes coordinates within interfaces without exchanging
distinct interfaces.
\end{definition}

\paragraph{Module exchanges.}
Some supported operations also permit exchanges of complete modules
and associated interfaces. Exchanging attention heads requires jointly
exchanging their query, key and value projections and corresponding
output-projection blocks. All affected typed incidences, fixed-state
constraints and shared-storage bindings must be preserved.

Let $E$ be the group generated by the supported exchange actions
for a compiled model. It permutes compatible factors of $\Gam$,
giving the combined action
\[
  G=\Gam\rtimes E.
\]
This extends $\Gam$ beyond within-interface permutations.
For $h$ interchangeable attention heads of query, key and value width
$d$, the internal discrete action is $(S_d\times S_d)^h\rtimes S_h$.

\begin{theorem}[Function preservation]
\label{thm:guarantee}
For every model accepted by the compiler under its supported operation,
layout, state-action and alias rules:
\begin{enumerate}
\item[(a)] Every coordinate relabelling in $\Gam$, applied consistently
through the bindings, leaves the network's function unchanged for every
admissible value of the represented state.
\item[(b)] Every supported module exchange in $E$, applied to all affected
modules, interfaces and stored-state blocks, leaves the network's function
unchanged.
\end{enumerate}
Consequently, every transformation in $G$ preserves the network's function.
\end{theorem}

Part (a) follows by composing the supported operations' equivariance,
enforcing fixed-state and alias constraints
(Appendix~\ref{app:soundness}); part (b) follows from the supported
exchange rules. Composition extends these actions to $G$.
Numerical tests separately check their implementation on compiled models.

\begin{remark}[Exactness and completeness]
\label{rem:complete}
$\Gam$ is \emph{exact} relative to the compiler's constraints:
it contains precisely the coordinate relabellings satisfying them
(Theorem~\ref{thm:compiler}, Appendix~\ref{app:compiler-guarantees}).
Soundness establishes function preservation, not the converse.
Transformations outside $\Gam$ may also preserve the function,
including continuous attention symmetries and arbitrary translations
of the key bias (Appendix~\ref{app:symm-cont}).
\end{remark}

\subsection{Size}
\label{sec:shape}

Let $A$ count interfaces, $S=\sum_I n_I$ their coordinates, and $C$
cells: the parameter count after folding kernel elements and
normalization pairs. Sharing interfaces within stages can put $A$
below the layer count; grouped or multi-head layouts can put it above.
Typical graphs have a few high-degree interfaces, such as residual
streams, and many of degree two or three. Table~\ref{tab:counts}
lists examples of compiled networks alongside compiled graph sizes.

\begin{table}[t]
\centering
\small
\setlength{\tabcolsep}{4pt}
\begin{tabular}{@{}lrrrl@{}}
\toprule
Network & Parameters & $A$ & $S$ & Interface widths \\
\midrule
ResNet-18 (20 convolutions)
  & 11.7M & 12 & 2,880 & 64--512, three of each \\
ResNeXt-50 $32{\times}4$d (53 convolutions)
  & 25.0M & 1,029 & 19,008 & mostly 4--32 \\
\midrule
ViT-S/16  & 22.1M  & 157   & 28,032  & 384, 64, 1536 \\
ViT-B/16  & 86.6M  & 301   & 56,064  & 768, 64, 3072 \\
ViT-L/16  & 304.3M & 793   & 148,480 & 1024, 64, 4096 \\
\midrule
GPT-2 Small  & 124.4M   & 301   & 56,064  & 768, 64, 3072 \\
GPT-2 Medium & 354.8M   & 793   & 148,480 & 1024, 64, 4096 \\
GPT-2 Large  & 774.0M   & 1,477 & 277,760 & 1280, 64, 5120 \\
GPT-2 XL     & 1,557.6M & 2,449 & 462,400 & 1600, 64, 6400 \\
\bottomrule
\end{tabular}
\caption{Permutation-bearing interfaces $A$ and coordinates $S$;
fixed interfaces excluded. A transformer with a permutation-bearing
residual interface, width $D$, $h$ heads and $L$ blocks has
$A=1+L(2h+1)$ and $S=D+L(2D+D_{\mathrm{ff}})$:
one residual interface of width $D$, two interfaces of width $D/h$
per head per block, and one feed-forward interface of width
$D_{\mathrm{ff}}$ per block. Configurations: Appendix~\ref{app:counts}.}
\label{tab:counts}
\end{table}

\subsection{Attention symmetries}
\label{sec:attention-example}

Table~\ref{tab:orders} in Appendix~\ref{app:attention-audit} compares the symmetries represented by each
method with attention's permutation group: independent query--key
and value permutations within heads, and exchanges of complete heads
\citep{tran2025tnfn,knyazev2024nino}.
The neural graph of \citet[Appendix~C.4]{kofinas2024} bundles
query, key and value entries into one edge feature and omits head
membership. It therefore permits function-changing cross-head
coordinate mixing while excluding independent value permutations.
Fixed head identifiers \citep{lim2024} prevent such mixing but
also forbid head exchange.
Transformer-NFN uses a hand-derived symmetry specification and
layers that scale cubically with embedding width
(Appendix~\ref{app:related}).
Our compiler recovers these permutation symmetries from supported
implementations, with an encoder whose parameter count is
independent of the input model's size.
\section{The encoder}
\label{sec:encoder}

The representation does not prescribe one encoder, but it fixes what any encoder over it must respect. We state the design at that level, keeping it as close to a standard transformer as the graph allows; the architecture used in our experiments is specified in full in Appendix~\ref{app:encoder}.

\paragraph{State.}
A rectangular unit $u$ bound to interfaces $I,J$ carries
$\E_u\in\R^{n_I\times n_J\times\dE}$; a diagonal unit, such as
a bias tensor, carries $\E_u\in\R^{n_I\times\dE}$.
Each cell encodes one scalar or a vector of folded parameters.
Interface $I$ carries $\Hs_I\in\R^{n_I\times\dH}$, one vector
per coordinate. Each weightless module $o$ carries
$u_o\in\R^{\dH}$, and $g$ learned global tokens
$\G\in\R^{g\times\dH}$ support invariant readouts.
Cells can greatly outnumber interface coordinates, so we use
narrower cell features, $\dE<\dH$, to reduce memory use.

\paragraph{Layer.}
Each encoder layer applies four passes (Figure~\ref{fig:encoder}).

\begin{itemize}[leftmargin=*]
\item \emph{Interface pass.}
Coordinate states are pooled into invariant interface summaries.
Local messages connect interfaces sharing a module, conditioned
on its type and binding roles. Global attention operates on the
summaries and optional global tokens.
Updates are broadcast to interface coordinates.

\item \emph{Delivery: interfaces to modules.}
Each unit gathers its bound interface states, adds receiving-role
embeddings, and projects them into row and column contributions
to its cells.

\item \emph{Module pass.}
Shared maps update cells from their features and pooled row and
column features, following equivariant weight-space layers
\citep{navon2023,zhou2023nfn,zhou2024unf}.
Weightless modules receive a feed-forward update. This is the analogue of the per-token feed-forward block in a conventional transformer.

\item \emph{Emission: modules to interfaces.}
Each unit pools its cells into one message per bound interface
coordinate, with sending-role embeddings.
Each coordinate averages incoming messages from its bound units
and weightless modules.
\end{itemize}

The small number of interfaces makes full attention practical
(Table~\ref{tab:counts}).

\paragraph{Codes.}
Learned embeddings identify parameter kinds, operation types and
binding roles. Numerical descriptors record attributes such as
stride and normalization $\epsilon$; sinusoidal features encode
module depth, ordered binding positions and offsets within folded
tensors. These codes condition the lift and local messages and
are shared across architectures (Appendix~\ref{app:encoder}).

\begin{figure}[t]
\centering
\resizebox{0.95\textwidth}{!}{%
\begin{tikzpicture}[>=Stealth, font=\footnotesize,
  box/.style={draw, rounded corners=2pt, inner sep=4pt, minimum height=7mm, align=center},
  tok/.style={draw=blue!60!black, fill=blue!25, minimum width=5mm, minimum height=5mm, inner sep=1pt}]
  \draw[fill=gray!12] (0,0.75) rectangle (2.0,-0.75);
  \foreach \i in {1,...,4} \draw (0.4*\i,0.75) -- (0.4*\i,-0.75);
  \foreach \i in {1,...,5} \draw (0,0.75-0.25*\i) -- (2.0,0.75-0.25*\i);
  \node[anchor=south] at (1.0,0.85) {$u$: conv or linear, $\E_u\in\R^{n_I\times n_J\times\dE}$};
  \draw[fill=gray!12] (0,-1.15) rectangle (1.6,-2.65);
  \foreach \i in {1,...,3} \draw (0.4*\i,-1.15) -- (0.4*\i,-2.65);
  \foreach \i in {1,...,5} \draw (0,-1.15-0.25*\i) -- (1.6,-1.15-0.25*\i);
  \node[anchor=north] at (0.8,-2.75) {$u'$: another module bound to $I$};
  \draw[fill=gray!12] (0,-3.3) rectangle (0.4,-4.8);
  \foreach \i in {1,...,5} \draw (0,-3.3-0.25*\i) -- (0.4,-3.3-0.25*\i);
  \node[anchor=north] at (0.9,-4.9) {$u''$: norm or bias, diagonal, $\E_{u''}\in\R^{n_I\times\dE}$};
  \draw[fill=blue!10] (4.6,-1.15) rectangle (5.2,-2.65);
  \foreach \i in {1,...,5} \draw (4.6,-1.15-0.25*\i) -- (5.2,-1.15-0.25*\i);
  \node[anchor=south, text width=3.2cm, align=center] at (4.9,-0.2) {interface $I$, $n_I$ coordinates: $\Hs_I\in\R^{n_I\times\dH}$};
  \draw[<->] (2.1,0.0) -- (4.5,-1.35);
  \draw[<->] (1.7,-1.9) -- node[above,font=\scriptsize]{emission $\rightarrow$ \quad $\leftarrow$ delivery} (4.5,-1.9);
  \draw[<->] (0.5,-4.05) -- (4.5,-2.45);
  \node[tok] (s) at (7.0,-1.9) {$s_I$};
  \draw[->] (5.3,-1.9) -- node[above,font=\scriptsize]{pool} (s);
  \node[box, minimum width=3.6cm] (att) at (10.0,-1.9) {local graph step and\\ global attention over summaries};
  \node[tok] (s2) at (8.0,-0.8) {$s_J$};
  \node[tok] (s3) at (8.0,-3.0) {$s_K$};
  \node[tok, fill=green!20] (g) at (10.0,-0.5) {$\G$};
  \draw[->] (s) -- (att); \draw[->] (s2) -- (att); \draw[->] (s3) -- (att); \draw[<->] (g) -- (att);
  \draw[->] (att.south) |- (4.9,-3.5) -- (4.9,-2.75);
  \node[font=\scriptsize, anchor=north] at (8.0,-3.6) {the interface update is broadcast to every coordinate of $I$};
\end{tikzpicture}}
\caption{One layer of the encoder. Every module bound to interface $I$ has an axis of exactly
$n_I$ cells, one per coordinate, whether rectangular ($u$, $u'$) or diagonal ($u''$); emission
pools a module's other axis into one message per coordinate, and delivery broadcasts the
coordinate states back along it. The global pass runs over one summary per interface and its result is broadcast to
all coordinates.}
\label{fig:encoder}
\end{figure}
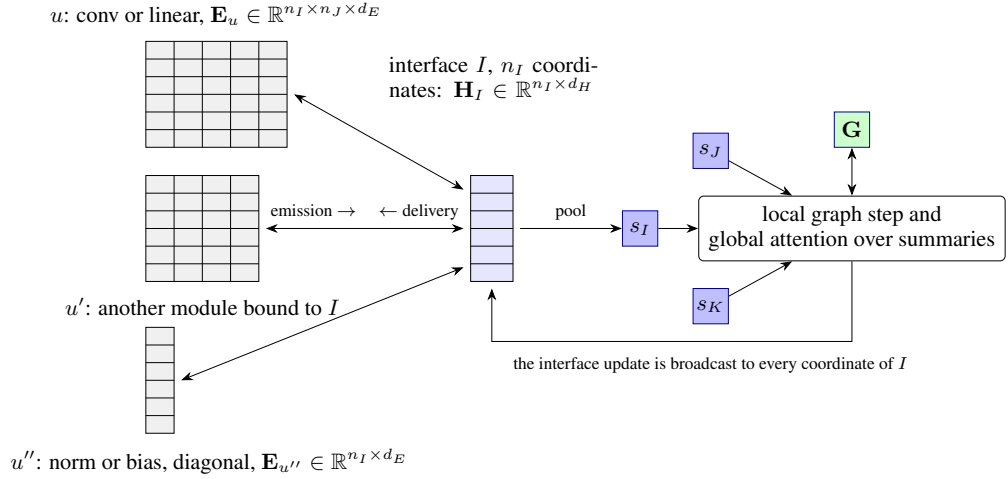

\paragraph{Lift.}
Scalar weights are encoded with a fixed multi-octave sine--cosine
basis \citep{tancik2020fourier,vaswani2017attention}, optionally
including the raw value, without rescaling.
A cell contains one scalar or a folded vector, such as kernel
elements or a normalization pair, with codes for their kinds
and offsets. The parameter layout retains every entry, while
the fixed-width cell encoding may compress folded vectors.
Module codes are added to cell features.

\paragraph{Fixed coordinates.}
For every fixed interface, the lift adds projected sinusoidal
coordinate encodings to its bound cell axes.
This includes input and output interfaces and residual streams
fixed by externally supplied positional encodings.
Learned positional embeddings are parameters that permute jointly
with the residual stream; permutable axes receive no coordinate
encoding.

\paragraph{Aggregation, normalization, and batching.}
Pooling uses shared maps and means \citep{zaheer2017};
LayerNorm acts only along the feature dimension.
Different architectures are batched as disjoint unions, with
attention and readouts separate for each network.
Matching graphs and metadata permit batching along a shared
batch dimension.
Encoder maps are shared across modules and architectures.

\paragraph{Equivariance.} Invariant network-level readouts $r$ and equivariant
parameter-valued readouts $F$ satisfy
\[
  r(g\theta)=r(\theta),
  \qquad
  F(g\theta)=gF(\theta),
  \qquad g\in G.
\]
Appendix~\ref{app:equivariance} gives the argument.

\section{The compiler}
\label{sec:compiler}

The compiler constructs the graph of Section~\ref{sec:representation}
from a supported PyTorch model and example inputs. It traces the
model in evaluation mode and resolves coordinate axes from tensor
shapes, dtypes and operator rules, without inspecting numerical
parameter values. External input and output coordinates are fixed.
Rules for parameters, buffers and shared storage specify their
permitted transformations; unsupported cases are rejected.

Each operation imposes constraints on coordinate permutations.
A \emph{glue} requires corresponding coordinate blocks to use the
same permutation: activations and normalization glue their input
and output axes, addition glues its inputs to its output, and
attention glues query and key coordinates within each head.
A \emph{split} requires permutations to preserve a block boundary,
such as an input boundary in coordinate-axis concatenation or
a head boundary in attention. A linear layer imposes no input--output glue.

The compiler propagates splits through glues until no new split
is introduced. This procedure terminates and is independent of
processing order. It then partitions each axis at the split
positions and identifies the resulting intervals whose permutations
must agree. Each equivalence class defines an interface, which is
fixed if any member belongs to a fixed axis. Parameters and buffers
are partitioned into units and bound to these interfaces.

The permutations satisfying the emitted constraints are exactly
$\Gam$. Supported head permutations generate the additional
exchanges $E$. Function preservation follows from the operator
rules under their stated assumptions. An independent checker
verifies the interface derivation against the emitted constraints;
numerical tests separately check the implementation.
Appendix~\ref{app:compiler} gives the compiler pipeline and interval
solver in Algorithms~\ref{alg:compile} and~\ref{alg:interval-closure}. Appendix~\ref{app:vocab} lists the supported operations and their
interface bindings.

\section{Results}
\label{sec:results}

Our focus is a shared representation and encoder for networks
of varying architecture and size. Standard fixed-architecture
benchmarks assess predictive quality against established methods.
The transformer zoo tests transfer across architectures, while
separate resource measurements assess computational feasibility
on larger networks.

\subsection{Standard homogenous benchmarks}

\begin{wraptable}{r}{0.56\textwidth}
\vspace{-\intextsep}
\centering
\small
\setlength{\tabcolsep}{4pt}
\caption{Test accuracy for INR classification and Kendall's $\tau$ for
CIFAR-10-GS CNN-zoo generalisation prediction. Higher is better; baseline
results are as published.}
\label{tab:wsl_results}
\begin{tabular}{@{}lccc@{}}
\toprule
& \multicolumn{2}{c}{INR accuracy (\%)} & CNN zoo \\
\cmidrule(lr){2-3}\cmidrule(lr){4-4}
Method & MNIST & F-MNIST & $\tau$ \\
\midrule
DWSNet & $85.71{\scriptstyle\,\pm0.57}$ & $67.06{\scriptstyle\,\pm0.29}$ & -- \\
NFN$_{\mathrm{NP}}$ & $78.50{\scriptstyle\,\pm0.23}$ & $68.19{\scriptstyle\,\pm0.28}$ & $0.922{\scriptstyle\,\pm0.001}$ \\
NFN$_{\mathrm{HNP}}$ & $79.11{\scriptstyle\,\pm0.84}$ & $68.94{\scriptstyle\,\pm0.64}$ & $0.934{\scriptstyle\,\pm0.001}$ \\
NG-GNN & $94.70{\scriptstyle\,\pm0.30}$ & $74.20{\scriptstyle\,\pm0.40}$ & $0.931{\scriptstyle\,\pm0.002}$ \\
NG-T & $97.30{\scriptstyle\,\pm0.20}$ & $74.80{\scriptstyle\,\pm0.90}$ & $0.935{\scriptstyle\,\pm0.000}$ \\
ScaleGMN & $96.57{\scriptstyle\,\pm0.10}$ & $80.46{\scriptstyle\,\pm0.32}$ & $\mathbf{0.941}{\scriptstyle\,\pm0.006}$ \\
ScaleGMN-B & $96.59{\scriptstyle\,\pm0.24}$ & $80.78{\scriptstyle\,\pm0.16}$ & $\mathbf{0.941}{\scriptstyle\,\pm0.000}$ \\
\midrule
UWE (ours) & $\mathbf{98.18}{\scriptstyle\,\pm0.26}$ & $\mathbf{85.33}{\scriptstyle\,\pm0.19}$ & $0.933{\scriptstyle\,\pm0.002}$ \\
\bottomrule
\end{tabular}
\vspace{-\intextsep}
\end{wraptable}

We evaluate three fixed-architecture tasks. INR classification
predicts image labels from the weights of SIRENs fitted to MNIST
or Fashion-MNIST images. CNN-zoo generalisation prediction estimates
CIFAR-10 CNN test accuracy, evaluated by Kendall's $\tau$
\citep{zhou2023nfn,kofinas2024}.

UWE achieves the highest INR accuracies among the compared methods:
98.18\% on MNIST and 85.33\% on Fashion-MNIST, compared with 97.30\%
and 80.78\% for the strongest respective baselines
(Table~\ref{tab:wsl_results}). On the CNN zoo, UWE reaches
$\tau=0.933$, compared with the best reported value of 0.941.

\subsection{Transformer zoo and loss profile prediction}
\label{sec:zoo}

\paragraph{New model zoo.}
Standard benchmarks largely fix the architecture
\citep{unterthiner2020,navon2023,tran2025tnfn}.
To provide a more challenging evaluation, we introduce
architectural variation and predict class-conditional loss
profiles under clean and corrupted inputs.
The benchmark tests transfer to unseen architectures,
including architectures larger than those used for training.

The zoo contains 1,200 CIFAR-100 vision transformers:
100 architectures sampled by log-parameter-count stratification
over 0.11M--9.1M parameters, each trained under 12 recipes crossing
four training-data fractions with three learning rates.
All recipes use augmentation, fixed weight decay and 10,560
optimiser steps. Each network is represented by its
best-validation-loss checkpoint, with a profile of mean
cross-entropy for each of 20 superclasses under clean, blur
and additive-noise conditions (60 targets).
Appendix~\ref{app:transformer-zoo} details the architecture family,
recipes and labelling protocol.

\begin{table}[t]
\centering
\small
\setlength{\tabcolsep}{4pt}
\begin{tabular}{@{}lrrrrrrrrr@{}}
\toprule
& & \multicolumn{3}{c}{common, common} & \multicolumn{2}{c}{common, beyond}
  & \multicolumn{3}{c}{full, full} \\
\cmidrule(lr){3-5}\cmidrule(lr){6-7}\cmidrule(lr){8-10}
Model & Params & MSE & $\tau$ & h & MSE & $\tau$ & MSE & $\tau$ & h \\
\midrule
Predict the training mean & --- & 0.9720 & --- & --- & 1.5510 & --- & 1.2860 & --- & --- \\
\midrule
NG-T, node 16, edge 4 & 72k & 0.9519 & 0.2165 & 23.1$^{\mathrm{H}}$ & --- & --- & --- & --- & --- \\
Transformer-NFN, $(c = 4)$ & 1.18B & 0.8206 & 0.3420 & 2.9$^{\mathrm{H}}$ & --- & --- & --- & --- & --- \\
GMN, width 64 & 0.19M & 0.9594 & 0.3766 & 1.5$^{\mathrm{H}}$ & 1.4273 & 0.4015 & --- & --- & --- \\
NG-GNN, released readout & 0.20M & 0.5263 & 0.5022 & 2.7 & 1.0320 & 0.5795 & 0.4007 & 0.6990 & 13.7 \\
NG-GNN, all-node readout & 0.10M & 0.4411 & 0.5455 & 2.9 & 1.1494 & 0.5114 & 0.3637 & 0.6970 & 13.8 \\
UWE, $\dE{=}16$, $\dH{=}32$ & 73k & 0.2790 & 0.7100 & 8.5 & 0.5361 & 0.7159 & 0.3355 & 0.7192 & 21.7 \\
UWE, $\dE{=}32$, $\dH{=}256$ & 3.46M & \textbf{0.2228} & \textbf{0.7965} & 8.9 & \textbf{0.5089} & \textbf{0.7538} & \textbf{0.3178} & \textbf{0.7778} & 23.0 \\
\bottomrule
\end{tabular}
\caption{\footnotesize Test results by (training cohort, test cohort),
at checkpoints selected by validation MSE.
MSE averages errors over 60 training-standardised profile coordinates;
$\tau$ ranks networks by mean profile loss within each architecture,
then averages over architectures.
h: runtime in hours; $^{\mathrm{H}}$: H100, otherwise A10.
Common, beyond reuses the common, common models, with no separate
runtime reported.
---: the model does not fit the cohort's architectures. Results use one metamodel training seed per configuration.
Variability across training seeds was not estimated due to
computational constraints.}
\label{tab:zoo-cohorts}

\end{table}

\paragraph{Experimental setup.}
Architectures are split 70/15/15 for training/validation/testing,
keeping each architecture's 12 networks together
(Appendix~\ref{app:zoo-splits}).
The \emph{common} cohort contains 43 training, seven validation
and seven test architectures supported by all baselines under
the memory constraints. Methods share its training and validation
networks; Table~\ref{tab:zoo-resource} reports resource costs.
The \emph{beyond-common} cohort contains the remaining eight test
architectures (96 networks, approximately 2.8--9.1M parameters).
All exceed the common training maximum of 2.12M, so evaluation
here tests transfer to larger architectures.
Full-cohort experiments use the complete 70/15/15 split.

We report MSE over 60 profile coordinates standardised using
training-set statistics. For ranking, we compute Kendall's $\tau$
between predicted and true mean profile losses across each test
architecture's 12 networks, then average over architectures.
Baselines use released implementations of NG-GNN and NG-T
\citep{kofinas2024}, GMN \citep{lim2024}, and Transformer-NFN
\citep{tran2025tnfn}.
All configurations use 100 epochs and the same learning rate;
validation MSE selects checkpoints.
UWE uses two sizes; the smaller approximately matches NG-T,
the smallest baseline, in parameter count.

\paragraph{Results.}
On the common test cohort, UWE reaches $\tau=0.7965$ versus
0.5455 for the strongest baseline (Table~\ref{tab:zoo-cohorts}).
The 73k-parameter UWE reaches 0.7100 versus 0.2165 for the
similarly sized NG-T.
On beyond-common architectures, the two UWE sizes reach 0.7538
and 0.7159 versus 0.5795 for the strongest baseline;
the larger UWE also reduces MSE from 1.0320 to 0.5089.
Full-cohort results change both training and test cohorts,
so they do not isolate the effect of adding larger training
architectures.
UWE uses less memory but takes about $3\times$ as long as
NG-GNN to train on the common cohort and $1.6\times$ on
the full cohort (A10).

\subsection{Ablations}
\label{sec:ablations}

\begin{wraptable}{r}{0.56\textwidth}
\vspace{-\intextsep}
\centering
\small
\setlength{\tabcolsep}{3pt}
\captionsetup{font=footnotesize,skip=4pt}
\caption{Common-trained ablations evaluated at their
best-validation-MSE checkpoints. Metrics follow
Table~\ref{tab:zoo-cohorts}.}
\label{tab:zoo-ablation}
\begin{tabular}{@{}lrrrrr@{}}
\toprule
& & \multicolumn{2}{c}{common} & \multicolumn{2}{c}{beyond} \\
\cmidrule(lr){3-4}\cmidrule(lr){5-6}
Variant & Params & MSE & $\tau$ & MSE & $\tau$ \\
\midrule
UWE & 73k & 0.28 & \textbf{0.71} & \textbf{0.54} & \textbf{0.72} \\
Without codes & 67k & \textbf{0.27} & 0.67 & 0.75 & 0.66 \\
Without global attention & 56k & 0.73 & 0.33 & 1.15 & 0.44 \\
Without either & 50k & 0.84 & 0.38 & 1.40 & 0.05 \\
\bottomrule
\end{tabular}
\vspace{-\intextsep}
\end{wraptable}

We ablate codes and global attention in a $2\times2$ design (Table~\ref{tab:zoo-ablation}, details in Appendix~\ref{app:ablation-setup}).
Removing global attention substantially worsens
MSE and ranking.
Removing codes slightly improves common-cohort MSE but lowers
$\tau$ and considerably worsens both beyond-common metrics.

\subsection{Scaling to large networks}
\label{sec:scaling-main}

We measure the memory required to process MLP, CNN, residual-CNN
and transformer targets with 1M--100M parameters
(Table~\ref{tab:scaling-summary}).
These experiments use random target weights and synthetic labels:
they assess computational cost and feasibility, not predictive
performance on trained networks.

UWE separates cell width $\dE$ from the wider interface states,
allowing encoder capacity to increase with modest memory growth.
Across three metamodel sizes, it completes training steps on every
target in the ladder and uses less memory than the parameter-matched
graph baselines wherever they fit. Transformer-NFN is compared at
two channel counts; its parameter count depends on target width.
On native ViT and GPT-2 targets with up to 355M parameters (Table~\ref{tab:app-native} in Appendix~\ref{app:scaling-native}),
UWE also completes training steps, while the matched graph
baselines run out of memory during inference.
Appendices~\ref{app:budget-ladder} and~\ref{app:scaling-native}
give the protocols and additional measurements.

\begin{table}[H]
\centering
\scriptsize
\setlength{\tabcolsep}{3pt}
\begin{tabular}{@{}lr rrr rrr rrr rr@{}}
\toprule
& & \multicolumn{3}{c}{UWE}
& \multicolumn{3}{c}{NG-GNN}
& \multicolumn{3}{c}{GMN}
& \multicolumn{2}{c}{Transformer-NFN} \\
\cmidrule(lr){3-5}\cmidrule(lr){6-8}
\cmidrule(lr){9-11}\cmidrule(lr){12-13}
Target & $P$
& 0.2M & 0.9M & 1.8M
& 0.2M & 0.9M & 1.8M
& 0.2M & 0.9M & 1.8M
& $c{=}2$ & $c{=}10$ \\
\midrule
MLP & 1M
& 0.46 & 0.54 & 0.59
& 2.28 & 5.11 & 7.21
& 3.81 & 8.19 & 11.39
& -- & -- \\
MLP & 30M
& 10.34 & 10.70 & 10.97
& 66.16 & OOM$_{\mathrm{i}}$ & OOM$_{\mathrm{i}}$
& OOM$_{\mathrm{t}}$ & OOM$_{\mathrm{i}}$ & OOM$_{\mathrm{i}}$
& -- & -- \\
CNN & 10M
& 0.55 & 0.63 & 0.69
& 2.63 & 5.80 & 8.09
& 37.47 & OOM$_{\mathrm{t}}$ & OOM$_{\mathrm{i}}$
& -- & -- \\
CNN & 100M
& 4.30 & 4.54 & 4.72
& 25.47 & 56.77 & OOM$_{\mathrm{t}}$
& OOM$_{\mathrm{i}}$ & OOM$_{\mathrm{i}}$ & OOM$_{\mathrm{i}}$
& -- & -- \\
Residual CNN & 100M
& 4.20 & 4.52 & 4.79
& 25.49 & 56.82 & OOM$_{\mathrm{t}}$
& -- & -- & --
& -- & -- \\
Transformer & 1M
& 0.57 & 0.79 & 0.97
& 2.21 & 4.95 & 6.93
& 3.99 & 8.57 & 11.92
& 1.31 & 6.41 \\
Transformer & 3M
& 1.37 & 1.70 & 1.97
& 6.43 & 14.54 & 20.41
& 11.76 & 25.40 & 35.36
& 6.44 & 32.23 \\
Transformer & 10M
& 3.78 & 4.40 & 4.88
& 20.59 & 46.70 & 65.62
& 37.81 & OOM$_{\mathrm{t}}$ & OOM$_{\mathrm{i}}$
& 36.96 & OOM$_{\mathrm{t}}$ \\
\bottomrule
\end{tabular}
\caption{\footnotesize Peak allocated memory (GiB) for a metamodel training step,
including the first AdamW step: one H100, 75\,GiB cap,
batch one.
Targets have random weights; predictive performance is not evaluated.
NG-GNN and GMN match UWE's trainable parameter count within 2.1\%.
Transformer-NFN is unmatched: at the 3M target, $c=2$ and $c=10$
have 381M and 1.924B parameters.
OOM$_{\mathrm{i}}$/OOM$_{\mathrm{t}}$: observed inference/training OOM;
--: unsupported or not applicable.}
\label{tab:scaling-summary}
\end{table}


\section{Limitations}
(i) The compiler's guarantees hold within its rule set and constraint
language, which exclude continuous changes of basis.
(ii) Unsupported constructs, including arbitrary attention masks,
cross-attention, rotary embeddings, recurrence and tensor sharing
without a supported alias rule, are rejected.
(iii) Compressing the parameters bound to each cell into a fixed-width
feature can lose information, particularly for large convolutional kernels.
(iv) Execution requires many small kernels, making deep targets
substantially slower per step than neural graph baselines.
Designing encoders that use modern hardware more efficiently
while respecting the representation's symmetries remains future work.
(v) The zoo covers one transformer family, one dataset and one task;
transfer across model families and tasks remains untested.
(vi) Resource measurements use random weights and establish computational
feasibility, not predictive quality.

\section{Discussion}

Organizing weight-space learning around shared permutation actions
gives a common representation for networks with different
architectures. Interfaces record which parameter axes must be
relabelled together, and modules describe the operations connecting
them. This structure serves two purposes: it specifies the
permutation action and provides a small graph for global
communication. UWE combines attention over interface summaries
with local equivariant updates to parameter features. The compiler
derives the required bindings from operator constraints, allowing
the same representation and encoder to handle new combinations
covered by its rules. The connection between symmetry,
architecture transfer and computational cost therefore follows
from the choice of what the graph represents.

The held-out-architecture results provide evidence that computations
shared at this level can learn useful predictors across changes
in transformer architecture. The resource measurements establish
feasibility on substantially larger networks; predictive quality
at that scale remains to be tested. More broadly, a shared
representation makes transfer possible, but does not ensure that
the learned features will remain useful across model families or
tasks. Training across several families and testing unseen
combinations of supported operations would help establish how
far interface-based learning extends beyond variation within a
transformer family.

\section*{AI use statement}
We used generative AI tools, including OpenAI Codex and Anthropic
Claude Code, to assist with implementation and debugging,
experimental design, data preparation, discussion of mathematical arguments, literature and reference checks,
and manuscript and figure preparation. These tools contributed
suggestions, code and draft text throughout the project. The authors
take full responsibility for the correctness, originality and
reporting of the work.

\section*{Reproducibility statement}
The compiler appendix specifies the supported operations, constraint
language and compilation algorithms. The technical appendix details
the encoder computations and the assumptions underlying the symmetry
guarantees. The experimental appendices describe the transformer zoo,
training procedures, data splits, prediction targets, baseline
configurations and resource measurements. These descriptions
distinguish architectural choices from experimental settings and
state the scope of the reported guarantees and comparisons. We release research code to reproduce the results and explore the proposed constructions. Developing a maintained Python package remains future work.

\section*{Ethics statement}
Learning from network parameters may reduce the cost of model
analysis and selection, although constructing model zoos and training
encoders also require computation. Performance predictions can be
unreliable outside the evaluated architectures and data
distributions, so their use in model selection requires validation
on the intended task. Predictive accuracy on the studied benchmarks
does not establish the fairness, privacy or robustness of the
networks being assessed.

\bibliography{refs}
\bibliographystyle{iclr2027_conference}

\appendix
\section{Related work}
\label{app:related}

Weight-space learning treats neural-network parameters as inputs
to learned models.
\citet{han2026survey} survey weight-space geometry, representation
learning and weight generation, together with their applications
and benchmarks.

\paragraph{Parameter-space symmetries.}
Hidden-unit permutations preserve the functions represented by
feedforward networks \citep{hechtnielsen1990,chen1993}.
Further symmetries depend on the activation, architecture and
parameter values \citep{godfrey2022,grigsby2023}.
\citet{zhao2025} review these symmetries and their applications.
Permutation alignment is used to merge independently trained
networks \citep{ainsworth2023}.
ScaleGMN incorporates both permutation and activation-dependent
scaling symmetries into graph metanetworks
\citep{kalogeropoulos2024}.

\paragraph{Equivariant tensor models.}
DWSNets construct permutation-equivariant layers for MLP weight
spaces \citep{navon2023}.
Neural functionals process MLP and CNN parameters through
equivariant tensor operations \citep{zhou2023nfn}, while neural
functional transformers introduce attention-based weight-space
processing \citep{zhou2023nft}.
Universal neural functionals automate equivariant layer
construction for collections of tensors with specified
permutation actions \citep{zhou2024unf}.
Their weight-space specification assigns shared permutation
identifiers to tensor axes, describing which axes must
transform jointly.

\paragraph{Graph metanetworks.}
\citet{kofinas2024} represent neurons as nodes and weights as
edge features, and process the resulting graphs with message
passing or relational transformers.
\citet{lim2024} construct parameter graphs for diverse neural
architectures, including architectural features such as residual
connections, normalization and parameter sharing.
They relate metanetwork equivariance to automorphisms of the
represented computation.
Message passing communicates along graph edges, whereas
relational transformers permit global interactions between nodes.
Invariant pooling provides network-level representations;
standard pooling constructions are studied by
\citet{zaheer2017}.

\paragraph{Transformer and adapter symmetries.}
Transformer-NFN studies the symmetries of multi-head attention
and constructs polynomial weight-space layers equivariant to
attention-head permutations and continuous query/key and value
basis changes \citep{tran2025tnfn}.
LoRA adapters have a related factorization ambiguity: invertible
transformations of the rank coordinates can preserve the
induced weight update.
W2T uses QR and SVD preprocessing to address this ambiguity,
followed by rank-level and position-level transformers
\citep{han2026w2t}.

\paragraph{Large and heterogeneous model collections.}
SANE processes windows of weight tokens to learn representations
of larger networks and generate network parameters
\citep{schurholt2024}.
\citet{falk2025} learn model representations from heterogeneous
networks collected from public model hubs, extending weight-space
learning beyond purpose-built model zoos.
These settings introduce variation in architecture, training data
and parameter count.

\paragraph{Expressivity.}
\citet{dayan2026} establish expressivity equivalences between
several permutation-equivariant weight-space architectures.
They derive universality results under assumptions on the input
weights and characterize settings where those results fail.
Their analysis distinguishes approximation of maps on parameter
space from approximation of functions of the represented
neural networks.
\section{Symmetries of neural network parameters}
\label{app:symmetries}

This appendix collects the symmetries a weight-space representation has to take a position on, and states which of them $\Gam$ (Definition~\ref{def:asserted}) contains. Surveys of the wider structure are given by \citet{zhao2025} and \citet{grigsby2023}.

Throughout, a network is a map $x\mapsto f(x;\theta)$ with
parameters $\theta\in\Theta$. A \emph{universal parameter symmetry}
is an invertible map $g:\Theta\to\Theta$ satisfying
$f(x;g\theta)=f(x;\theta)$ for every $x$ and $\theta$.
These transformations form a group under composition.
Additional equivalences may hold only at particular parameter values;
we refer to these as \emph{accidental}. Only universal symmetries are candidates for a representation to build in.

\subsection{Permutations}
\label{app:symm-perm}

Let $\sigma$ act coordinatewise. For a layer $h=\sigma(W_1x+b_1)$ followed by $y=W_2h+b_2$ and any $P\in\Sn{n}$,
\begin{equation}
  W_1\mapsto PW_1,\qquad b_1\mapsto Pb_1,\qquad W_2\mapsto W_2P^{\top}
  \label{eq:mlp-perm}
\end{equation}
leaves $y$ unchanged, because $\sigma(Pz)=P\sigma(z)$ (Figure~\ref{fig:symmetries}a). The same holds for a convolution acting on its channel axes, with the elements of a kernel carried along and never permuted among themselves (Figure~\ref{fig:symmetries}b): for $W\in\R^{c_{\mathrm{out}}\times c_{\mathrm{in}}\times k_hk_w}$,
\begin{equation}
  W[:,:,t]\mapsto P_{\mathrm{out}}\,W[:,:,t]\,P_{\mathrm{in}}^{\top}\quad\text{for every kernel element }t.
  \label{eq:cnn-perm}
\end{equation}
Kernel elements are index domains, not coordinate spaces: a permutation of them is a different convolution, not the same one relabelled. An $L$-layer perceptron of hidden widths $n_1,\dots,n_L$ therefore admits at least $\prod_l n_l!$ hidden-coordinate relabellings \citep{hechtnielsen1990,chen1993}, and for perceptrons with generic weights this group is the whole universal symmetry group up to activation-induced factors \citep{chen1993}. Permutation symmetry is what weight alignment and model merging exploit \citep{ainsworth2023}. Everything in $\Gam$ is of this kind; Section~\ref{app:vocab} gives the cases where architecture restricts which permutations are allowed.

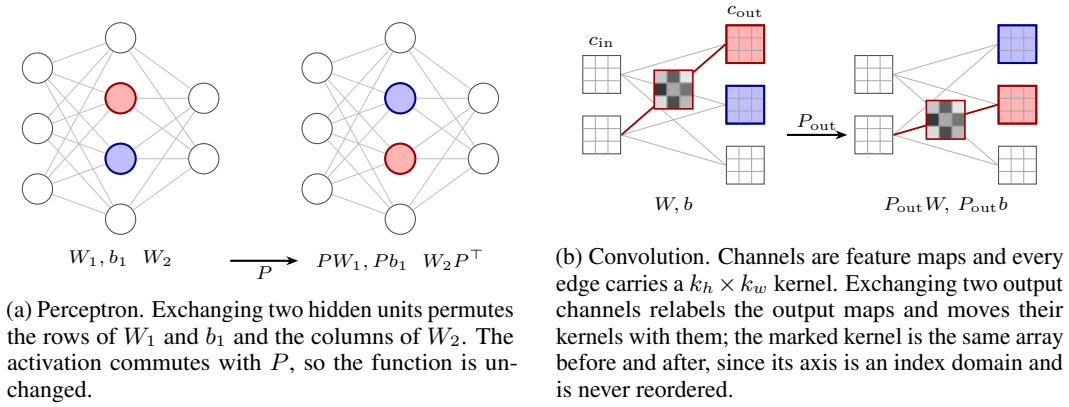
\begin{figure}[t]
\centering
\tikzset{
  nd/.style={circle, draw=black!70, fill=white, minimum size=4.0mm, inner sep=0pt},
  hA/.style={nd, fill=red!30, draw=red!60!black, thick},
  hB/.style={nd, fill=blue!25, draw=blue!60!black, thick},
  ed/.style={draw=black!25, line width=0.25pt},
  tile/.style={draw=black!45, fill=gray!8, line width=0.3pt},
  chA/.style={draw=red!60!black, fill=red!30, thick},
  chB/.style={draw=blue!60!black, fill=blue!25, thick},
  lab/.style={font=\scriptsize, inner sep=1.5pt},
  ar/.style={>={Stealth[length=1.6mm]}, thick}}
\begin{subfigure}[b]{0.48\textwidth}
\centering
\begin{tikzpicture}[x=1cm, y=1cm]
\foreach \s/\dx in {0/0, 1/3.7}{
  \foreach \i in {1,2,3} \node[nd] (x\s\i) at (\dx,2.2-0.8*\i) {};
  \foreach \i in {1,2} \node[nd] (y\s\i) at (\dx+2.2,1.8-0.8*\i) {};
}
\node[nd] (h01) at (1.1,1.8) {};
\node[hA] (h02) at (1.1,1.0) {};
\node[hB] (h03) at (1.1,0.2) {};
\node[nd] (h04) at (1.1,-0.6) {};
\node[nd] (h11) at (4.8,1.8) {};
\node[hB] (h12) at (4.8,1.0) {};
\node[hA] (h13) at (4.8,0.2) {};
\node[nd] (h14) at (4.8,-0.6) {};
\foreach \s in {0,1}{
  \foreach \i in {1,2,3} \foreach \j in {1,2,3,4} \draw[ed] (x\s\i) -- (h\s\j);
  \foreach \j in {1,2,3,4} \foreach \i in {1,2} \draw[ed] (h\s\j) -- (y\s\i);
}
\draw[ar,->] (2.55,-1.15) -- node[lab, below] {$P$} (3.45,-1.15);
\node[lab, anchor=north] at (1.1,-0.95) {$W_1,b_1$ \ \ $W_2$};
\node[lab, anchor=north] at (4.8,-0.95) {$PW_1,Pb_1$ \ \ $W_2P^{\top}$};
\end{tikzpicture}
\caption{Perceptron. Exchanging two hidden units permutes the rows of $W_1$ and $b_1$ and the columns of $W_2$. The activation commutes with $P$, so the function is unchanged.}
\label{fig:symm-mlp}
\end{subfigure}
\hfill
\begin{subfigure}[b]{0.48\textwidth}
\centering
\begin{tikzpicture}[x=1cm, y=1cm,
  fm/.style={draw=black!70, fill=white, line width=0.4pt},
  fmg/.style={black!35, line width=0.25pt},
  kc/.style={line width=0.2pt, draw=black!40}]
\useasboundingbox (-0.35,-1.15) rectangle (5.85,2.0);
\foreach \dx/\rr/\bb in {0/1.3/0.5, 3.6/0.5/1.3}{
  \foreach \yi in {0.9,0.1} \foreach \yo in {1.3,0.5,-0.3}
    \draw[black!30, line width=0.3pt] (\dx+0.26,\yi) -- (\dx+1.64,\yo);
  \draw[red!65!black, line width=0.7pt] (\dx+0.26,0.1) -- (\dx+1.64,\rr);
  \foreach \y in {0.9,0.1}{
    \draw[fm] (\dx-0.25,\y-0.25) rectangle ++(0.5,0.5);
    \foreach \k in {1,2}{
      \draw[fmg] (\dx-0.25+0.1667*\k,\y-0.25) -- ++(0,0.5);
      \draw[fmg] (\dx-0.25,\y-0.25+0.1667*\k) -- ++(0.5,0);}
  }
  \draw[fm, fill=red!30, draw=red!60!black, line width=0.7pt] (\dx+1.65,\rr-0.25) rectangle ++(0.5,0.5);
  \draw[fm, fill=blue!25, draw=blue!60!black, line width=0.7pt] (\dx+1.65,\bb-0.25) rectangle ++(0.5,0.5);
  \foreach \y in {1.3,0.5,-0.3}{
    \draw[fm, fill=none] (\dx+1.9-0.25,\y-0.25) rectangle ++(0.5,0.5);
    \foreach \k in {1,2}{
      \draw[fmg] (\dx+1.9-0.25+0.1667*\k,\y-0.25) -- ++(0,0.5);
      \draw[fmg] (\dx+1.9-0.25,\y-0.25+0.1667*\k) -- ++(0.5,0);}
  }
  \draw[red!60!black, line width=0.7pt] (\dx+1.65,\rr-0.25) rectangle ++(0.5,0.5);
  \draw[blue!60!black, line width=0.7pt] (\dx+1.65,\bb-0.25) rectangle ++(0.5,0.5);
  \pgfmathsetmacro{\my}{(0.1+\rr)/2}
  \foreach \cx/\cy/\g in {0/0/22, 1/0/62, 2/0/10, 0/1/78, 1/1/34, 2/1/54, 0/2/14, 1/2/70, 2/2/28}
    \filldraw[kc, fill=black!\g] (\dx+0.95-0.255+0.17*\cx, \my+0.255-0.17*\cy) rectangle ++(0.17,-0.17);
  \draw[red!65!black, line width=0.6pt] (\dx+0.95-0.255,\my-0.255) rectangle ++(0.51,0.51);
}
\node[lab, anchor=south] at (0,1.2) {$c_{\mathrm{in}}$};
\node[lab, anchor=south] at (1.9,1.6) {$c_{\mathrm{out}}$};
\node[lab, anchor=north] at (0.95,-0.66) {$W,b$};
\node[lab, anchor=north] at (4.55,-0.66) {$P_{\mathrm{out}}W,\,P_{\mathrm{out}}b$};
\draw[ar,->] (2.45,0.1) -- node[lab, above] {$P_{\mathrm{out}}$} (3.25,0.1);
\end{tikzpicture}
\caption{Convolution. Channels are feature maps and every edge carries a $k_h\times k_w$ kernel. Exchanging two output channels relabels the output maps and moves their kernels with them; the marked kernel is the same array before and after, since its axis is an index domain and is never reordered.}
\label{fig:symm-cnn}
\end{subfigure}
\caption{Permutation symmetry acts on coordinate spaces and carries the parameters along the axes bound to them. Colour marks the two exchanged coordinates.}
\label{fig:symmetries}
\end{figure}

\subsection{Activation-induced symmetries}
\label{app:symm-act}

A permutation is available whatever the activation, as long as it acts coordinatewise. Particular activations admit more \citep{godfrey2022}. If $\sigma$ is positively homogeneous, $\sigma(\alpha z)=\alpha\sigma(z)$ for $\alpha>0$, as ReLU and leaky ReLU are, then for any positive diagonal $D$
\begin{equation}
  W_1\mapsto DW_1,\qquad b_1\mapsto Db_1,\qquad W_2\mapsto W_2D^{-1}
  \label{eq:rescale}
\end{equation}
preserves the function, giving a factor $\R_{>0}^{n}$ per hidden layer, semidirect with the permutations. If $\sigma$ is odd, as $\tanh$ and $\sin$ are, the same holds with $D$ restricted to $\{\pm1\}^{n}$. Sinusoidal representations admit a second sign symmetry that needs no absorption downstream, since $\sin(\pi-z)=\sin(z)$ gives $(w,b)\mapsto(-w,\pi-b)$ on any unit. Smooth non-homogeneous activations such as GELU, SiLU and the logistic function admit no nontrivial scaling: $\sigma(\alpha z)=\alpha\sigma(z)$ for all $z$ forces $\alpha=1$.

Normalization requires its own symmetry rules. Channel permutations
commute with batch normalization when its channelwise affine
parameters and running statistics are permuted consistently.
Layer normalization over a coordinate axis likewise respects
permutations of that axis when any affine parameters transform with
it. This does not imply invariance to arbitrary independent
rescaling of the channels.

We restrict $\Gam$ to coordinate permutations; activation-dependent
rescalings are outside this scope. Scale-equivariant graph
metanetworks explicitly incorporate such actions
\citep{kalogeropoulos2024}. Independently of this choice, UWE does
not rescale weights on input (Section~\ref{sec:encoder}), so their
magnitudes remain available to the encoder.

\subsection{Continuous symmetries from contraction}
\label{app:symm-cont}

A third family arises wherever two parameter blocks enter the function only through their product. If $f$ depends on matrices $\theta_1\in\R^{m\times r}$ and
$\theta_2\in\R^{r\times n}$ only through their product, an invertible
change of basis in the intermediate space preserves the function:
\[
  (\theta_1,\theta_2)\longmapsto
  (\theta_1G^{-1},G\theta_2),
  \qquad G\in\mathrm{GL}(r).
\]

Three instances matter in current practice. In an attention head with $W_Q,W_K\in\R^{d\times D}$, the logits are $XW_Q^{\top}W_KX^{\top}/\sqrt{d}$, so the head depends on $W_Q$ and $W_K$ only through $W_Q^{\top}W_K$ and
\begin{equation}
  W_Q\mapsto G^{-\top}W_Q,\qquad W_K\mapsto GW_K,\qquad G\in\mathrm{GL}(d),
  \label{eq:qk-gl}
\end{equation}
preserves the function; biases transform with their projections. The same head depends on $W_V\in\R^{d\times D}$ and its block $W_O\in\R^{D\times d}$ of the output projection only through $W_OW_V$, giving a second $\mathrm{GL}(d)$ \citep{tran2025tnfn}. In a LoRA adapter $W+\tfrac{\alpha}{r}BA$ with $B\in\R^{d_{\mathrm{out}}\times r}$ and $A\in\R^{r\times d_{\mathrm{in}}}$, only $BA$ is observable, so $\mathrm{GL}(r)$ acts on the rank space; this is the ambiguity that W2T removes by QR and SVD preprocessing before encoding \citep{han2026w2t}.

A degenerate case of the same phenomenon is a parameter that is unobservable outright. Adding any vector to the key bias $b_K$ shifts every logit of a query row by the same amount, and softmax is invariant to a row shift, so $b_K$ can be changed arbitrarily \citep{namazifar2023}. The compiler transforms $b_K$ with its head's query--key permutation, which is sound but not maximal (Remark~\ref{rem:complete}).

\begin{table}[h]
\centering
\small
\setlength{\tabcolsep}{4pt}
\begin{tabular}{@{}llll@{}}
\toprule
Symmetry & Origin & Group & In $\Gam$ \\
\midrule
hidden-unit permutation
  & coordinatewise activation
  & $\Sn{n}$
  & yes \\
channel permutation
  & convolution on channel axes
  & $\Sn{c}$
  & yes \\
within-group permutation
  & grouped-convolution layout
  & $\prod_j\Sn{c/g}$
  & yes \\
head exchange
  & exchange of complete heads
  & $\Sn{h}$
  & no; included in $G$ \\
\midrule
positive rescaling
  & positively homogeneous $\sigma$
  & $\R_{>0}^{n}$
  & no \\
sign flip
  & odd $\sigma$
  & $\{\pm1\}^{n}$
  & no \\
\midrule
query--key basis change
  & logits depend on $W_Q^{\top}W_K$
  & $\mathrm{GL}(d)$
  & permutation subgroup \\
value--output basis change
  & output depends on $W_OW_V$
  & $\mathrm{GL}(d)$
  & permutation subgroup \\
LoRA factor basis change
  & update depends on $BA$
  & $\mathrm{GL}(r)$
  & permutation subgroup \\
key-bias translation
  & softmax row-shift invariance
  & $\R^{d}$
  & no  \\
\midrule
parameter-dependent equivalences
  & e.g.\ dead or duplicate units
  & N.A.
  & no\\
\bottomrule
\end{tabular}
\caption{Parameter symmetries and their treatment in this paper.
The last column refers to the within-interface group $\Gam$.
A permutation action is included only when the corresponding interfaces
are permutation-bearing and the model's bindings permit it.
\emph{Permutation subgroup} means that $\Gam$ contains the permitted
coordinate permutations within the stated continuous group, without
asserting the remaining transformations.
Supported head exchanges belong to the enlarged group
$G=\Gam\rtimes E$.
The attention product expressions describe bias-free projections;
affine projections require the corresponding bias transformations.}
\label{tab:symmetries}
\end{table}

\subsection{Accidental symmetries}
\label{app:symm-accidental}

Some symmetries depend on the parameter values. For example, the
outgoing weights of a ReLU unit that is identically inactive on the
input domain may change without affecting the function. Identical
hidden units also permit redistribution of their outgoing weights
while preserving their sum. The compiler does not attempt to recover
these parameter-dependent equivalences.

\subsection{Summary}
\label{app:symm-summary}

Table~\ref{tab:symmetries} summarizes the transformations considered
here. The group $\Gam$ contains precisely the within-interface
coordinate permutations permitted by the compiler's constraints.
The enlarged group $G$ additionally contains supported exchanges of
complete modules. Neither group is claimed to contain every
function-preserving transformation.

\section{The transformer zoo}
\label{app:transformer-zoo}

We construct a zoo of 1,200 transformers spanning 100 architectures
to evaluate weight-space learning across changes in width, depth,
head count and feed-forward ratio. Every block includes residual
connections around both sublayers. Each architecture is trained
under the same 12 recipes, varying data availability and learning rate.

\subsection{Source architecture family}
\label{app:zoo-arch}

Each source is a vision transformer for CIFAR-100. A $4\times4$
convolution with stride 4 maps a $32\times32$ image to 64 tokens
of width $d$, followed by fixed sinusoidal positional encodings.
The network has $L$ pre-normalization blocks, each containing
bias-free multi-head attention and a ReLU feed-forward network
of hidden width $rd$, with a residual connection around each
sublayer. All layer normalizations have no affine parameters.
A final normalization, mean pooling over tokens and a classifier
$d\to d\to100$, with an intermediate ReLU, produce the logits.
For compilation, the positional encodings are fixed external
inputs, so their addition fixes the residual-stream interface.
With $C=100$ classes, the learned parameter count is
\begin{equation}
  P(d,L,r)
  = \big(1+L(4+2r)\big)d^2
  + \big(50+C+L(r+1)\big)d+C,
  \label{eq:zoo-params}
\end{equation}
excluding the fixed positional encodings.

The candidate grid uses $L\in\{2,3,4,6\}$,
$d\in\{64,96,\ldots,640\}$, $r\in\{2,3,4\}$ and
$h\in\{2,3,4,6,8,12,16\}$ heads. We retain configurations with
$h\mid d$, head width $32\le d/h\le96$ and
$10^5\le P\le10^7$. To sample 100 distinct architectures, we
divide $[\log 10^5,\log 10^7]$ into 100 equal intervals and draw
one uniform target per interval. For each target, we sample
uniformly among unused candidates within $0.18$ in natural-log
parameter count, choosing the nearest unused candidate if none
qualifies. Sampling uses seed 926314 and no training outcomes.

\begin{table}[t]
\centering
\small
\setlength{\tabcolsep}{5pt}
\begin{tabular}{@{}ll@{}}
\toprule
Architectures & 100 distinct $(d,L,h,r)$ \\
Parameters & 112,676 to 9,061,732; median 1.05M \\
Width $d$ & 64 to 608 \\
Depth $L$ & 2 (30), 3 (31), 4 (23), 6 (16) \\
Heads $h$ & 2 to 16; head width 32 to 96 \\
FFN ratio $r$ & 2 (34), 3 (38), 4 (28) \\
\midrule
Runs & 1,200: one per recipe per architecture \\
Training subset sizes & 5,600; 11,200; 22,500; 45,000 images \\
Labels & 3 conditions $\times$ 20 superclass mean losses \\
Raw FP32 parameter storage & 9.75\,GiB \\
\bottomrule
\end{tabular}
\caption{Transformer zoo summary. Parentheses give architecture
counts at each depth or feed-forward ratio.}
\label{tab:zoo}
\end{table}

\subsection{Training}
\label{app:zoo-train}

We reserve 50 images per fine class from the CIFAR-100 training
set for source validation, using seed 921314. The remaining
45,000 images form the training pool and determine the shared
normalization statistics. The official 10,000-image test set
is used only to compute labels.

The 12 recipes cross nominal training-data fractions
$\{0.125,0.25,0.5,1\}$ with peak learning rates
$\{1.5,3,6\}\times10^{-4}$. Subsets are class-balanced,
deterministic and nested, and are shared across architectures
and learning rates; classwise rounding gives the sizes in
Table~\ref{tab:zoo}. Every recipe uses weight decay $10^{-2}$
and augmentation: a random $32\times32$ crop after reflect
padding by four pixels, followed by a random horizontal flip.

All runs use AdamW with batch size 256 for 10,560 optimizer
updates. The learning rate increases linearly for 1,056 updates,
then follows cosine decay to 1\% of its peak. We clip the
gradient norm at one and use bfloat16 autocast with float32
cross-entropy. A fixed update budget prevents smaller data
subsets from receiving fewer optimization steps. Source
validation is evaluated every 132 updates; the checkpoint
with the lowest clean validation loss supplies the weights
and labels. All 1,200 runs completed the budget without
skipped updates.

Initialization and data-order seeds are derived from each
run's identifier. The manifest records seeds, recipes and
splits; training snapshots retain the optimizer,
random-number-generator and data-order states for resumption.

\subsection{Labels}
\label{app:zoo-labels}

Each selected checkpoint is evaluated on the official test set
under three conditions: clean images; Gaussian blur with
$\sigma=1$ pixel, a $5\times5$ kernel and reflect padding;
and additive Gaussian noise with standard deviation $0.05$
on $[0,1]$ pixel values, clipped to that range. Perturbations
precede normalization. Noise is seeded by image index, so
every network receives the same corrupted images.
Evaluation uses float32 in evaluation mode.

The target comprises the mean cross-entropy within each of
the 20 CIFAR-100 superclasses under each condition, giving
60 values per network. The source still predicts 100 fine
classes: superclass grouping averages fine-label losses over
500 images rather than 100. The profile retains class and
condition differences that a single average would hide.
We also store the 300 fine-class means, per-condition mean
loss and accuracy, and per-example losses.
Figure~\ref{fig:zoo-v2} summarizes the resulting labels.

\begin{figure}[t]
\centering
\includegraphics[width=\linewidth]{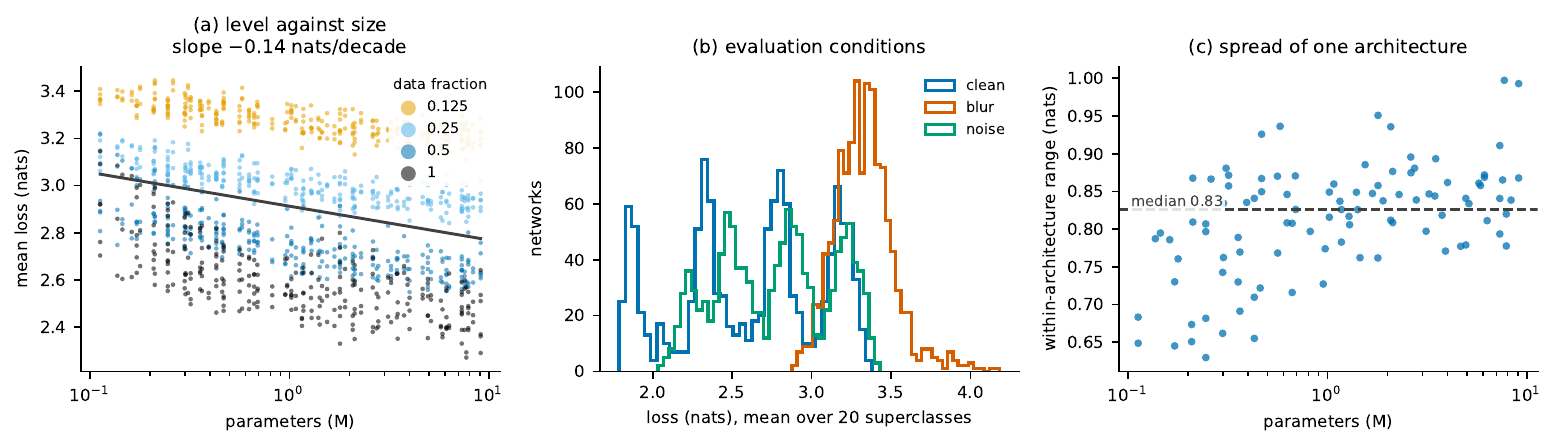}
\caption{Transformer zoo labels.
(a) Mean profile loss versus parameter count, coloured by
training-data fraction; the fitted slope is $-0.14$ nats per
decade.
(b) Distributions of per-condition mean loss.
(c) Range of mean profile loss across the 12 runs of each
architecture; the dashed line marks the median, $0.83$ nats.}
\label{fig:zoo-v2}
\end{figure}

\subsection{Split protocols}
\label{app:zoo-splits}

The manifest defines two protocols independently of training
outcomes. The \emph{held-out-architecture} protocol, used in
Section~\ref{sec:zoo}, assigns 70/15/15 architectures to
training, validation and testing. All runs follow their
architecture, giving 840/180/180 networks and preventing
architecture overlap between splits.

The \emph{familiar-architecture} protocol instead splits runs
within architectures, giving 832/187/181 networks. Every
evaluation architecture is represented in training; this
protocol evaluates generalization to new trained instances
of known architectures.

All experiments in the paper use the \emph{held-out-architecture} protocol.

\section{Experimental setup}
\label{app:experimental-setup}

This section specifies predictive training and evaluation.
Zoo construction is described in Appendix~\ref{app:transformer-zoo};
resource measurements are separate
(Appendix~\ref{app:scaling}).
All predictive runs use AdamW with
$(\beta_1,\beta_2)=(0.9,0.999)$ and $\epsilon=10^{-8}$.

\subsection{Standard benchmarks}
\label{app:standard-setup}

\paragraph{Data.}
We use the released INR and CIFAR-10-GS CNN benchmarks
\citep{navon2023,zhou2023nfn,kofinas2024}.
Each INR is a $2\to32\to32\to1$ SIREN with 1,185 parameters,
fitted to one MNIST or Fashion-MNIST image.
Each dataset contains 55,000/5,000/10,000 training/validation/test
networks.
The CNN zoo uses step-86 checkpoints of networks with three
$3\times3$, 16-channel convolutions and a linear classifier
(4,970 parameters).
The released shuffle and IID split give
12,000/3,000/15,000 networks.
Targets are image classes for INRs and recorded test accuracies
for CNNs.

\paragraph{Configurations.}
Table~\ref{tab:standard-setup} gives task-specific settings. The readout averages coordinates
within interfaces, and global tokens.
The concatenated summaries enter a
$2\dH\to2\dH\to C$ MLP with GELU, with $C=10$ for INRs
and $C=1$ for CNNs.
Source weights are not standardized. INRs use the sine--cosine lift in Equation~\ref{eq:fold}. CNN's lift folds kernel perameters as in Appendix~\ref{app:encoder}. UWE's parameter count roughly follows that of NG-GNN, roughly matching it on INRs, and being about a half of it on the CNN task.

\begin{table}[t]
\centering
\small
\setlength{\tabcolsep}{5pt}
\begin{tabular}{@{}lrrr@{}}
\toprule
Setting & MNIST-INR & Fashion-MNIST-INR & CNN zoo \\
\midrule
$\dE$ & 16 & 16 & 40 \\
$\dH$ & 80 & 80 & 160 \\
Layers & 4 & 4 & 8 \\
Attention heads & 4 & 4 & 8 \\
Global tokens & 4 & 4 & 4 \\
Learned probes & 64 & 64 & 0 \\
Batch size & 128 & 128 & 256 \\
Warm-up updates & 1,000 & 1,000 & 0 \\
Gradient clipping norm & 1 & 1 & 10 \\
Epochs & 300 & 100 & 300 \\
Trainable parameters & 341,082 & 341,082 & 2,744,421 \\
\bottomrule
\end{tabular}
\caption{Standard-benchmark settings.
Counts include the encoder, readout and probes. Hyperparameters were loosely chosen to align the model with baseline methods.}
\label{tab:standard-setup}
\end{table}

\paragraph{Probes and augmentation.}
Following the setup of \citet{kofinas2024}, INR runs use 64 jointly learned two-dimensional probes, initialized
uniformly in $[-1,1]^2$ and evaluated with $\omega_0=30$.
Activations provide coordinate-aligned observation features
(Appendix~\ref{app:providers}).

\paragraph{Training and evaluation.}
Runs use AMSGrad, learning rate $10^{-3}$ and weight decay
$5\times10^{-4}$.
The INR learning rate increases linearly for 1,000 updates,
then remains constant; CNN uses a constant rate throughout. Runs use FP32, PyTorch 2.13, Apple Silicon MPS,
four CPU threads and no data-loader workers.

INRs minimize cross-entropy.
CNN minimizes binary cross-entropy with logits against accuracy
targets in $[0,1]$, applying sigmoid at evaluation.
Validation runs every epoch; checkpoint selection maximizes
accuracy for INRs and Kendall's $\tau_b$ for CNNs. Baseline results in
Table~\ref{tab:wsl_results} are taken from the cited publications.

\subsection{Transformer loss-profile prediction}
\label{app:zoo-metamodels}

\paragraph{Cohorts.}
We use the held-out-architecture split
(Appendix~\ref{app:zoo-splits}).
All 12 networks of an architecture share its split.
The full cohort contains 70/15/15 training/validation/test
architectures, or 840/180/180 networks.

The common development cohort was selected by preflighting
UWE and NG-GNN under a 20\,GiB A10 cap,
NG-T under a 70\,GiB H100 cap, and a two-channel
Transformer-NFN bank padded to width 288 under the
20\,GiB A10 cap.
The intersection contains 50 development architectures.
The final configurations, including GMN, smaller UWE and the
four-channel bank, were subsequently checked under their
respective budgets.

The common cohort contains 43/7/7 architectures
(516/84/84 networks).
Beyond-common comprises the remaining eight test architectures
(96 networks), spanning approximately 2.8--9.1M parameters;
all exceed the common training maximum of 2.12M.
Common-trained models are evaluated on common and beyond-common
test cohorts separately.
Full-cohort models use the full training, validation and test
splits.
Partial-coverage results include the number of networks evaluated;
unsupported inputs and failed evaluations are recorded separately.

\paragraph{UWE.}
The first two rows of Table~\ref{tab:encoder-configurations} specify model configuration for the transformer zoo experiments.
All use four layers, four attention heads, and sine--cosine weight
features.

Both headline sizes use structural codes.
The readout applies learned LayerNorm to coordinate features,
averages coordinates within interfaces and interfaces equally,
then applies a $\dH\to2\dH\to60$ MLP with GELU.
Each source architecture is compiled once; its bindings are
reused across trained instances.

\begin{table}[t]
\centering
\small
\setlength{\tabcolsep}{5pt}
\begin{tabular}{@{}lrrrrcr@{}}
\toprule
Experiment & $\dE$ & $\dH$ & $g$
& Layers & Readout & Parameters \\
\midrule
Loss profile, small
& 16 & 32 & 0 & 4 & $H$ & 72,636 \\
Loss profile, large
& 32 & 256 & 0 & 4 & $H$ & 3,458,684 \\
\midrule
Budget ladder, small
& 16 & 56 & 4 & 4 & $H,G$ & 190,057 \\
Budget ladder, medium
& 16 & 128 & 4 & 4 & $H,G$ & 906,529 \\
Budget ladder, large
& 16 & 180 & 4 & 4 & $H,G$ & 1,759,277 \\
\midrule
Native ViT/GPT-2
& 4 & 128 & 4 & 4 & $H,G$ & 887,881 \\
Native ViT/GPT-2
& 8 & 128 & 4 & 4 & $H,G$ & 893,649 \\
Native ViT/GPT-2
& 16 & 128 & 4 & 4 & $H,G$ & 906,529 \\
\bottomrule
\end{tabular}
\caption{UWE configurations for loss-profile prediction and resource
measurements;
$g$ is the number of global tokens.
Readout uses pooled interface features $H$ and, when present,
global-token features $G$.
Counts include the encoder and prediction head.
The parameter-matched native comparison uses $\dE=4$.}
\label{tab:encoder-configurations}
\end{table}

\paragraph{NG-GNN and NG-T.}
We minimally adapt the released implementations of \citet{kofinas2024}.
The graph includes a patch linear map from 48 fixed coordinates,
position nodes, separate Q/K and V coordinate groups,
head membership, identity edges for non-affine normalization
and residual skip edges.
Biases enter node features; Q and K use separate edge channels.
Entrywise checks verify that every learned source parameter
appears once.

NG-GNN uses node and edge width 32,
edge updates and node-wise LayerNorm.
We compare concatenation of the 100 output-node features
(the released readout) with all-node mean pooling.
NG-T uses two relational-transformer layers, eight heads,
node width 16, edge width four and output-node concatenation.
Both retain learned positional and activation embeddings,
128-dimensional sinusoidal weight features and zero dropout.

Beyond the released configuration, we evaluate NG-GNN,
the strongest graph baseline, with an all-node readout
to test whether adapting its aggregation improves performance.
Both variants are reported.
Computational constraints precluded extensive
method-specific hyperparameter searches.

\paragraph{GMN.}
The released EdgeMPNN-DiT \citep{lim2024} uses four width-64
layers and a DeepSets edge readout.
We omit its final sigmoid for standardized regression.
The adapter combines released convolution, attention,
feed-forward and linear constructors.
Bias-free attention receives fixed zero bias vectors;
automatic residual edges are retained.
These contribute $4d$ zero-bias and $2d$ residual edges per block,
in addition to learned-parameter edges.
Head partitions, fixed positional encodings, non-affine
normalization, ReLU and token averaging are not explicitly encoded.

\paragraph{Transformer-NFN.}
The bank uses released Transformer-NFN layers
\citep{tran2025tnfn}, six block slots and maximum width 288, with zero padding.
The encoder has one four-channel equivariant stage,
followed by the released invariant construction. Blocks use LayerNorm.
Patch-embedding, classifier and encoder branches each produce
ten features, concatenated and mapped linearly to 60 outputs. Width 288 balances coverage and bank size.
Wider sources require a larger bank and are unsupported by this
configuration. This construction already has 1.18B parameters.

\begin{table}[t]
\centering
\small
\setlength{\tabcolsep}{6pt}
\begin{tabular}{@{}lr@{}}
\toprule
Metamodel & Trainable parameters \\
\midrule
UWE, $\dE=16$, $\dH=32$ & 72,636 \\
UWE, $\dE=32$, $\dH=256$ & 3,458,684 \\
NG-GNN, output-node concatenation & 203,516 \\
NG-GNN, all-node mean & 102,140 \\
NG-T, node 16, edge 4 & 72,388 \\
GMN, width 64 & 185,340 \\
Transformer-NFN, width-288 four-channel bank & 1,179,934,498 \\
\bottomrule
\end{tabular}
\caption{Headline parameter counts, including input encodings
and prediction heads.
Small UWE approximately matches NG-T;
the full comparison is not parameter-matched.}
\label{tab:zoo-metamodels}
\end{table}

\paragraph{Objective and optimization.}
Each target coordinate is standardized using its training-cohort
mean and population standard deviation, with equal total weight
per architecture.
These statistics are fixed for validation and testing, and
recomputed for common versus full training.
The objective is MSE over 60 standardized coordinates,
weighted equally across architectures.
Each architecture has 12 networks.
Input weights are unstandardized; no weight or permutation
augmentation is applied.

All models use seed 715 and a 100-epoch budget.
AdamW uses no AMSGrad, peak learning rate $3\times10^{-4}$
and weight decay $10^{-4}$.
One warm-up epoch precedes cosine decay towards 5\% of the peak.
Each epoch visits all training networks in shuffled order.
Physical batch size is one, with gradient accumulation over
16 networks. Gradients are clipped at norm one.
Common/full cohorts give 33/53 optimizer updates per epoch.
Validation runs every epoch; selection minimizes
architecture-averaged validation MSE.

Runs use FP32, PyTorch 2.11 and CUDA 12.8.
UWE and NG-GNN use A10s with a 20\,GiB allocator cap;
GMN, NG-T and Transformer-NFN use H100s with a 70\,GiB cap due to memory demand slightly above A10s capcity.

\paragraph{Metrics.}
For each network, MSE is the mean squared error over its 60
training-standardized targets. For ranking, we convert predictions back to the original loss
units and average each predicted and true profile over its 60
coordinates. Within each test architecture, we compute Kendall's $\tau_b$
between the predicted and true scores of its 12 networks.
We report the mean of these architecture-level coefficients.
The training-mean baseline predicts the same profile for every
network, so its Kendall correlation is undefined and shown as a dash.

\subsection{Ablations}
\label{app:ablation-setup}

The $2\times2$ common-cohort ablation crosses structural codes
with global attention.
All variants use small UWE's widths, with the same training protocol as the transformer zoo experiment.

\paragraph{Code removal.}
Each parameter-kind, operation-type and sending or receiving
role table has one shared learned embedding.
Descriptors and ordered-binding contributions are zeroed,
depth encodings are removed, and folded offsets use a constant
$(0,0)$ geometry code.
Fixed-coordinate encodings remain. As most parameters do not require folding at lift (unlike for example in CNNs, where geometric codes provide positional information within a kernel), this ablation does not
isolate convolutional offset information.

\paragraph{Attention removal.}
We remove attention between interface summaries. No variant uses global tokens.
This measures the contribution of global interface-to-interface communication, which the interface-module graph representation makes feasible at scale.

Parameter counts are 72,636 for full UWE, 67,116 without codes,
55,740 without global attention, and 50,220 without either.

\section{Resource measurements}
\label{app:scaling}

These experiments measure metamodel memory and execution time,
not predictive performance.
Zoo measurements use trained source weights; ladder and native
measurements use random source weights.
All training-step measurements use synthetic losses.

\subsection{Transformer zoo}
\label{app:zoo-resource}

These measurements support the resource and architecture-coverage
discussion in Section~\ref{sec:zoo}.
Table~\ref{tab:zoo-resource} reports the results.

\paragraph{Protocol.}
We use the configurations, representations and AdamW settings
from Appendix~\ref{app:zoo-metamodels}. For each of 85 train and validation architectures, the first available
trained network supplies the input.
A step predicts the 60 outputs, minimizes their mean square,
backpropagates, clips gradients at norm one and updates the
metamodel.
Each metamodel is instantiated once per sweep.
Packing and transfer are timed separately and excluded.
Peak allocated memory is reset per architecture and includes
inputs, metamodel parameters, gradients and optimizer state.

\paragraph{Hardware and failures.}
Measurements use PyTorch 2.11 and CUDA 12.8.
OOM includes observed allocation failures and, on A10, recorded
allocator/NVML errors under memory pressure.

\begin{table}[t]
\centering
\small
\setlength{\tabcolsep}{3.5pt}
\begin{tabular}{@{}lrrrrrrrr@{}}
\toprule
& & \multicolumn{2}{c}{Fits, of 85}
& \multicolumn{4}{c}{Peak GiB by target size, H100}
& Step \\
\cmidrule(lr){3-4}\cmidrule(lr){5-8}
Method & Params & H100 & A10
& 0.1--0.3M & 0.3--1M & 1--3M & 3--10M & (s) \\
\midrule
UWE, $\dE=32$, $\dH=256$
& 3.46M & 85 & 85 & 0.72 & 1.43 & 3.24 & 9.17 & 0.81 \\
UWE, $\dE=16$, $\dH=32$
& 73k & 85 & 85 & 0.22 & 0.46 & 1.13 & 3.19 & 0.82 \\
NG-GNN, width 32
& 0.20M & 85 & 85 & 0.54 & 1.60 & 4.32 & 12.91 & 0.10 \\
GMN, width 64
& 0.19M & 85 & 69 & 1.66 & 5.13 & 14.58 & 44.14 & 0.15 \\
NG-T, node 16, edge 4
& 72k & 52 & 26 & 21.93$^{17}$ & 57.54 & 68.47$^{12}$ & -- & 1.47 \\
Transformer-NFN, four-channel bank
& 1.18B & 65 & NR & 22.05 & 22.05 & 22.05$^{22}$ & 22.06$^{2}$ & 0.09 \\
\bottomrule
\end{tabular}
\caption{Zoo resource use: 
Configurations follow Table~\ref{tab:zoo-metamodels} with batch one.
NG-GNN uses the released output-node concatenation readout.
The 85 architectures span widths 64--608 and 0.11M--8.3M parameters.
Fits counts completed steps under 70\,GiB H100 and 20\,GiB A10 caps.
Peak is the maximum over successful H100 cases per bin
(18/23/25/19 architectures); superscripts give partial-coverage
counts.
Step is the median over each method's successful H100 cases.
GMN and NG-T failures are OOMs except one NG-T implementation
failure; Transformer-NFN coverage is width-limited.
--: no completed measurement; NR: not evaluated.}
\label{tab:zoo-resource}
\end{table}
\subsection{Measurement protocol for scaling experiments}
\label{app:scaling-protocol}

This protocol applies to the budget ladder reported in
Section~\ref{sec:scaling-main} and Table~\ref{tab:scaling-summary},
and to the native-model experiments in
Appendix~\ref{app:scaling-native}.
The zoo measurements follow Appendix~\ref{app:zoo-resource}.

Ladder and native targets use random initialization with a fixed
seed.
Each case runs in a separate process on one H100 80\,GB:
75\,GiB allocator cap, four CPU threads, PyTorch 2.11,
CUDA 12.8, FP32, autocast and TF32 disabled, batch one.

A training step comprises squared error against a fixed synthetic
target, backward propagation, gradient clipping at norm one and
AdamW with learning rate $10^{-4}$, and zero weight decay.
Inference and training are measured separately.
We first perform one forward pass or complete training step and
record its peak allocated GPU memory.
We then perform two warm-up repetitions followed by five timed
repetitions, synchronizing the GPU before and after each.
Reported time is the median of the five timed repetitions.
Reported memory is the largest peak allocated GPU memory across
the first repetition and the five timed repetitions, including
optimizer-state initialization for training.
Target generation, representation construction and transfer are
excluded and measured separately; learned input encoding is included.

Successful training requires finite loss and a finite, non-zero
gradient norm.
OOMs retain their phase (inference or training step); timeouts and implementation errors
are separate.

\subsection{Budget ladder}
\label{app:budget-ladder}

\paragraph{Targets.}
We scale MLP, CNN, residual CNN and Transformer families
to approximately 1M, 3M, 10M, 30M and 100M parameters.
MLPs have four linear layers with ReLU.
CNNs have three convolutions; residual CNNs have a convolutional
stem and three two-convolution residual blocks.
Both use $3\times3$ kernels, ReLU, global average pooling and
a linear classifier.

Transformer cores have two blocks, eight heads, bias-free
attention, LayerNorm and ReLU feed-forward layers
of width $4d$, without residual connections or positional encoding.
Widths are the smallest meeting the nominal count, rounded to
multiples of eight for transformers.

\paragraph{Parameter matching.}
UWE, NG-GNN and GMN use four rounds at approximately
0.2M, 0.9M and 1.8M metamodel parameters.
UWE follows Table~\ref{tab:encoder-configurations}, with
190,057, 906,529 and 1,759,277 parameters.

NG-GNN uses released PNA layers \citep{kofinas2024},
sparse execution, output-node concatenation and graph LayerNorm.
Node and edge widths are 45/103/146 for MLPs,
46/104/146 for CNNs and residual CNNs, and
45/103/145 for transformers.
GMN \citep{lim2024} uses widths 66/144/201,
DeepSets edge readout and the released sigmoid.

Counts include input encoding, embeddings, backbone and readout.
Graph widths minimize the difference from each UWE budget;
the largest mismatch is 2.042\%.
Transformer-NFN \citep{tran2025tnfn} uses two or ten channels
and is not parameter-matched.

\subsection{Complete ViT and GPT-2 targets}
\label{app:scaling-native}

\paragraph{Targets.}
We use torchvision \texttt{VisionTransformer}:
ViT-B/16 has width 768, 12 blocks and 12 heads;
ViT-L/16 has width 1,024, 24 blocks and 16 heads.
Both retain $224\times224$ inputs, $16\times16$ patches,
a class token, learned positions and a 1,000-class head.

Hugging Face \texttt{GPT2LMHeadModel} supplies GPT-2 Small
(768, 12, 12) and Medium (1,024, 24, 16), with vocabulary
size 50,257, 1,024 learned positions and tied token-embedding
and output weights.
UWE traces one image or eight tokens for shape propagation.
Complete parameter tensors are retained; shared parameters
are encoded once.

\paragraph{Metamodels.}
UWE follows Table~\ref{tab:encoder-configurations}, with
$\dH=128$, $\dE\in\{4,8,16\}$ and cell chunks of 65,536.
Graph baselines match the $\dE=4$ model's 887,881 parameters:
NG-GNN width 94 for ViT and 17 for GPT-2; GMN width 143.
The largest mismatch is 1.984\%.
Coordinate embeddings dominate NG-GNN's GPT-2 parameter count,
so total-count matching does not equalize message-passing capacity.

NG-GNN uses node-wise LayerNorm and all-node mean pooling;
GMN omits its final sigmoid.
These differ from the ladder's graph LayerNorm,
output-node concatenation and retained GMN sigmoid.

Methods share targets and the protocol in
Appendix~\ref{app:scaling-protocol}, with a 1,200\,s case limit.
Table~\ref{tab:app-native} combines the matched $\dE=4$
comparison with earlier $\dE=8,16$ measurements.

\begin{table}[t]
\centering
\small
\setlength{\tabcolsep}{5pt}
\begin{tabular}{@{}lrlrrrr@{}}
\toprule
Target & $P$ (M) & Encoder
& Infer.\ (s) & Infer.\ (GiB) & Step (s) & Step (GiB) \\
\midrule
ViT-B/16 & 86.57 & UWE, $\dE=4$
& 2.66 & 3.64 & 17.95 & 17.72 \\
& & UWE, $\dE=8$
& 2.54 & 5.93 & 17.37 & 24.35 \\
& & UWE, $\dE=16$
& 2.57 & 11.12 & 17.66 & 37.39 \\
& & NG-GNN
& OOM & -- & -- & -- \\
& & GMN
& OOM & -- & -- & -- \\
\midrule
ViT-L/16 & 304.33 & UWE, $\dE=4$
& 7.83 & 12.49 & 54.60 & 58.19 \\
& & UWE, $\dE=8$
& 7.38 & 20.57 & OOM & -- \\
& & UWE, $\dE=16$
& 7.32 & 38.80 & OOM & -- \\
& & NG-GNN
& OOM & -- & -- & -- \\
& & GMN
& OOM & -- & -- & -- \\
\midrule
GPT-2 Small & 124.44 & UWE, $\dE=4$
& 3.12 & 5.09 & 22.81 & 22.27 \\
& & UWE, $\dE=8$
& 2.91 & 8.53 & 22.29 & 31.82 \\
& & UWE, $\dE=16$
& 2.92 & 16.03 & 22.25 & 51.00 \\
& & NG-GNN
& OOM & -- & -- & -- \\
& & GMN
& OOM & -- & -- & -- \\
\midrule
GPT-2 Medium & 354.82 & UWE, $\dE=4$
& 8.29 & 14.37 & 61.37 & 63.75 \\
& & UWE, $\dE=8$
& 7.79 & 23.98 & OOM & -- \\
& & UWE, $\dE=16$
& 7.93 & 45.31 & OOM & -- \\
& & NG-GNN
& OOM & -- & -- & -- \\
& & GMN
& OOM & -- & -- & -- \\
\bottomrule
\end{tabular}
\caption{Native-model resource measurements: one H100,
75\,GiB cap, FP32, batch one, four rounds and checkpointing.
UWE has 0.888M/0.894M/0.907M parameters at $\dE=4/8/16$,
with $\dH=128$.
Graph baselines match $\dE=4$ within 2\%.
Memory includes the initial forward or AdamW step;
time is the median of five repetitions.
OOM: failure in the indicated phase;
--: no completed measurement.
Random source weights are used; predictive performance
is not measured.}
\label{tab:app-native}
\end{table}
\section{The compiler}
\label{app:compiler}

The compiler constructs the representation of
Section~\ref{sec:representation} from a supported PyTorch model and example
inputs. It traces the model in evaluation mode and uses tensor shapes,
dtypes, operation attributes and parameter-sharing information to infer
interfaces and parameter bindings. This inference does not depend on
numerical parameter values. Parameters and buffers are subsequently
assigned to the corresponding units.

The compiler recovers the permutation group defined by its constraints.
Function preservation additionally depends on the supported operator,
layout and parameter-sharing rules.

\subsection{Overview}
\label{app:compiler-walkthrough}

\paragraph{Coordinate axes.}
The compiler records the computational graph and identifies the feature
or channel dimension of each relevant tensor. We call this its
\emph{coordinate axis}. Linear layers specify the final feature dimension;
convolutions specify the channel dimension. Layout operations, such as
transposes, relate coordinate axes across tensors. Axis constraints
propagate in both directions through the graph.

\paragraph{Permutation constraints.}
Each operation imposes constraints on its input and output permutations.
A linear layer without parameter sharing allows independent input and
output permutations, with corresponding permutations of its weights and
bias. Elementwise addition requires the same permutation on both inputs
and the output. Concatenation along the coordinate axis requires the output
permutation to preserve the input blocks and agree with each input
permutation within its block.

The compiler represents these requirements using three constraints:
a \emph{glue} equates permutations on two coordinate intervals;
a \emph{split} prevents permutations across a boundary; and a
\emph{fixed} axis permits only the identity.

\paragraph{Interface inference.}
The compiler propagates splits through glues until no new split is
introduced. It then partitions each coordinate axis into intervals between
consecutive split positions, called \emph{atoms}. Atoms identified by glues
form equivalence classes. Each class defines an interface with one shared
permutation. If any atom belongs to a fixed axis, the entire class is fixed.

\paragraph{Example.}
Let $h_1,h_2,h_3$ be the outputs of three linear layers with dimensions
$2$, $3$ and $5$, respectively:
\[
\begin{aligned}
 h_i&=W_i x+b_i,\qquad i=1,2,3,\\
 c&=\operatorname{concat}(h_1,h_2),\qquad
 s=c+h_3,\qquad y=W_o s+b_o.
\end{aligned}
\]
The layers do not share parameters, and batch dimensions are omitted.
Concatenation introduces a split at position~2 of $c$. Addition propagates
this split to $h_3$ and $s$. Their first two coordinates share an interface
with $h_1$; their remaining three share an interface with $h_2$. The two
interfaces admit independent permutations, giving $S_2\times S_3$.
The input and output coordinates remain fixed. The same partition splits
the rows of $W_3$ and the columns of $W_o$ into the corresponding units. See Figure~\ref{fig:algo-example} for a schematic illustration.

\begin{figure}[t]
    \centering
    \includegraphics[width=\linewidth]{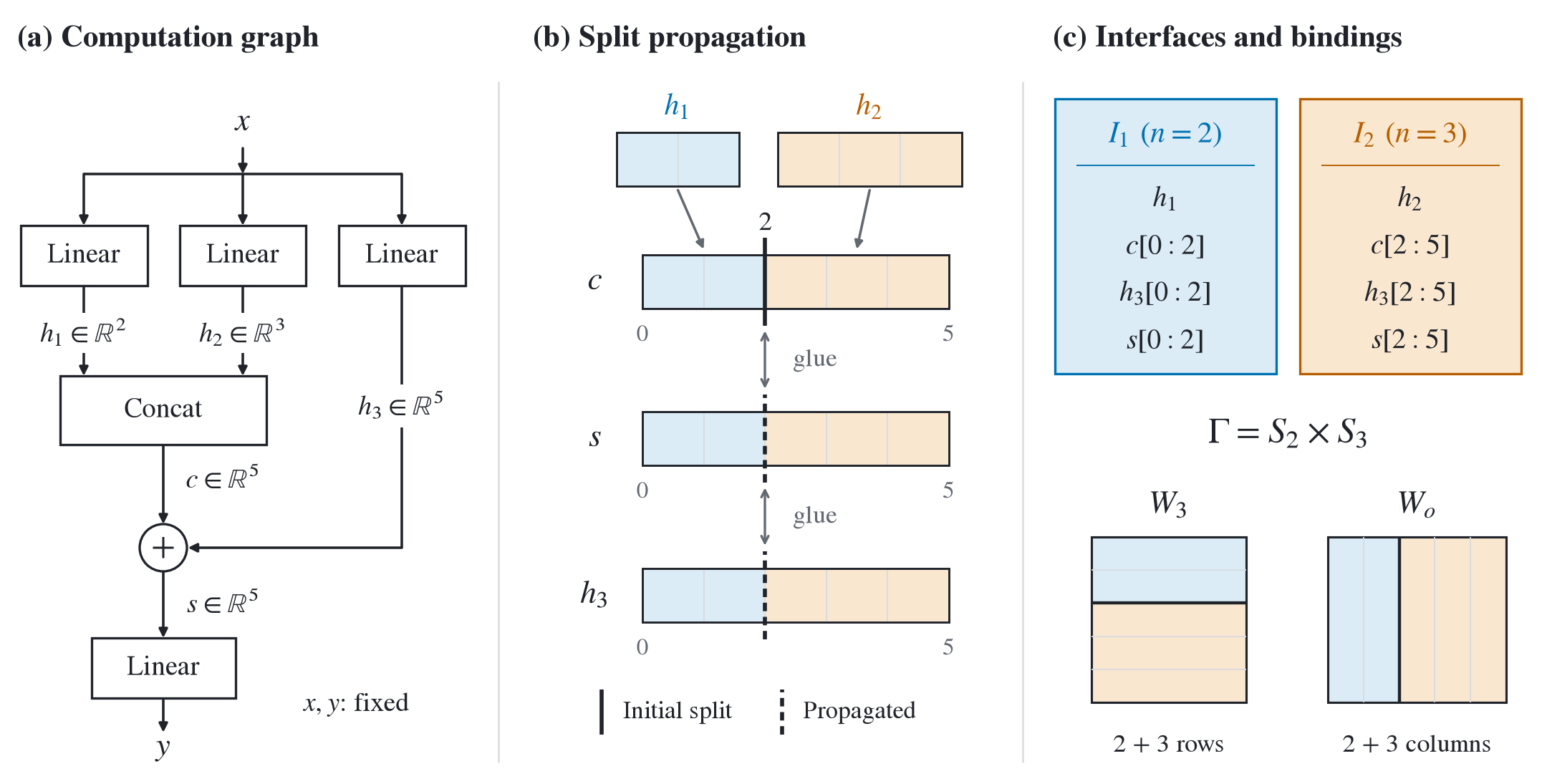}
    \caption{Interface inference for concatenation and addition.
    Concatenation splits $c$ at position~2; addition propagates this
    split to $s$ and $h_3$. The resulting atoms form two interfaces
    with permutation group $S_2\times S_3$. Colors indicate their
    bindings to rows of $W_3$ and columns of $W_o$. Input and output
    coordinates remain fixed.}
    \label{fig:algo-example}
\end{figure}

\paragraph{Parameter bindings and head permutations.}
The inferred interfaces determine the permutations of parameters and
buffers. Binding rules specify the affected tensor dimensions and slices,
including constraints from parameter sharing. Attention-head permutations
are handled separately: a head permutation must also permute the
corresponding query, key and value projections and the matching blocks of
the output projection.

\subsection{Constraints and operator rules}
\label{app:compiler-constraints}

Let $\mathcal X$ be a finite set of coordinate axes. Axis $t$ has width
$n_t\ge1$ and coordinates $[n_t]=\{0,\ldots,n_t-1\}$. Axes may correspond
to traced tensors or be introduced by an operator rule. In particular,
attention rules introduce internal query, key and value axes. Integer
indices and supported shape computations do not receive coordinate
permutation factors.

The constraint system is $\mathcal C=(\mathcal X,n,S,L,F)$, where $S$
contains splits, $L$ contains glues and $F\subseteq\mathcal X$ contains
fixed axes. A permutation assignment is a tuple
$(\pi_t)_{t\in\mathcal X}$ with $\pi_t\in S_{n_t}$.
Intervals below contain integer coordinates.

\begin{itemize}
\item A split $(t,c)$, with $0\le c\le n_t$, prohibits permutations
      across position $c$:
      \[
        \pi_t(i)<c \iff i<c,
        \qquad i\in\{0,\ldots,n_t-1\}.
      \]
\item A glue $(t,u;\,t',u';\,k)$ relates two blocks of $k\ge1$
      coordinates, with $0\le u\le n_t-k$ and
      $0\le u'\le n_{t'}-k$.
      It imposes splits at $u,u+k$ on $t$ and at $u',u'+k$
      on $t'$. Both blocks must use the same permutation
      $\sigma$ of $\{0,\ldots,k-1\}$:
      \[
      \begin{aligned}
        \pi_t(u+j)   &= u+\sigma(j),\\
        \pi_{t'}(u'+j) &= u'+\sigma(j),
      \end{aligned}
      \qquad j=0,\ldots,k-1.
      \]
      The endpoint splits are part of the constraint: both permutations
      preserve the specified interval and agree in relative coordinates.
\item A fixed axis $t\in F$ requires $\pi_t=\mathrm{id}$.
\end{itemize}

\paragraph{Axis inference.}
Operator rules specify coordinate-axis positions, and layout rules
propagate these positions in both directions. Conflicting assignments
are rejected. Supported reshapes preserve the coordinate axis without
combining it with another non-singleton dimension.

If propagation permits several positions for an unassigned tensor, the
implementation selects the lowest index. Components without an assigned
axis are initialized at tensor dimension~1, using zero-based indexing,
and propagation resumes. Tensors for which this assignment is unavailable
or inconsistent are rejected. The exactness theorem is conditional on
the resulting axis assignment.

\paragraph{Operator rules.}
Each rule specifies its layout requirements, permutation constraints and
parameter bindings. Rules may introduce internal axes and fixed axes.
External input and output coordinate axes are fixed; embedding rules also
fix the vocabulary axis. Parameter-sharing rules impose consistent
permutations at every supported use of a shared tensor.
Attention rules separately declare supported head permutations.

Table~\ref{tab:compiler-facts} summarizes the main constraint patterns.
Operator contracts, state bindings and their principal restrictions
are given in Appendix~\ref{app:operators}.

\begin{table}[t]
\centering
\small
\begin{tabular}{@{}p{0.35\linewidth}p{0.60\linewidth}@{}}
\toprule
Operation & Constraints \\
\midrule
Linear layer without parameter sharing
& No input--output glue. \\

Elementwise operation; addition or multiplication with aligned
coordinate axes
& Glue the participating coordinate axes. \\

Concatenation along the coordinate axis
& Glue each input to its output interval and split at the interval
  boundaries. \\

Concatenation along another axis
& Glue each input coordinate axis to the output coordinate axis;
  no coordinate splits. \\

Depthwise convolution with
$g=n_{\rm in}=n_{\rm out}>1$
& Glue input and output coordinate axes. \\

Other grouped convolution, $g>1$
& Split input and output coordinate axes at group boundaries;
  no input--output glue. \\

Supported attention with $h$ heads of dimension $d$
& Introduce $q,k,v$ axes of width $hd$; glue $q$ to $k$ and split
  all three at head boundaries. Declare head permutations separately. \\
\bottomrule
\end{tabular}
\caption{Representative operator rules. Each row assumes a supported
layout. Fixed-axis and parameter-sharing constraints may further restrict
the permitted permutations.}
\label{tab:compiler-facts}
\end{table}

\subsection{Algorithms}
\label{app:compiler-algorithms}

Algorithm~\ref{alg:compile} specifies the compiler pipeline.
Operator-specific steps follow the rules above and reject unsupported
cases. Algorithm~\ref{alg:interval-closure} specifies interface inference
from the emitted constraints.

\begin{algorithm}[t]
\caption{Compile the representation}
\label{alg:compile}
\begin{algorithmic}[1]
\REQUIRE model $f$, example inputs $x$, rule set $V$
\STATE trace $f$ in evaluation mode and propagate shapes and dtypes
       using fake inputs
\STATE retain the computational graph $\mathcal T$ and typed
       operation occurrences $\mathcal O$
\STATE validate parameter sharing and record supported aliases $\mathcal A$
\STATE infer coordinate-axis positions $a$ using the operator and layout rules
\STATE $(\mathcal C,\mathcal D)\leftarrow
       \mathrm{EmitRules}_V(\mathcal T,a,\mathcal A)$
\STATE \hspace{1em}\textit{$\mathcal C$: axes and constraints;
       $\mathcal D$: head-permutation declarations}
\STATE $(C,\mathcal I,R)\leftarrow
       \mathrm{RecoverInterfaces}(\mathcal C)$
       \quad \textit{(Algorithm~\ref{alg:interval-closure})}
\STATE verify $C$, $\mathcal I$ and derivation record $R$
       against $\mathcal C$
\STATE partition parameters and buffers into units $\mathcal U$
       and construct bindings $\mathcal B$, respecting $\mathcal A$
\STATE add the supported attention-head and output-projection
       occurrences to $\mathcal O$
\STATE $N_{\rm expected}\leftarrow$ entry count of distinct registered
       parameters and buffers, excluding BatchNorm training counters
\STATE $N_{\rm packed}\leftarrow$ total entry count of the blocks in $\mathcal U$
\IF{$N_{\rm packed}\ne N_{\rm expected}$}
  \STATE reject: parameter and buffer coverage mismatch
\ENDIF
\STATE construct head permutations $\mathcal H$ from $\mathcal D$
       and the inferred representation
\STATE verify $\mathcal H$ against the declarations, occurrences,
       interfaces and bindings
\STATE return $(\mathcal I,\mathcal O,\mathcal U,\mathcal B,\mathcal H)$
       and $(\mathcal C,C,R)$
\end{algorithmic}
\end{algorithm}

For a glue $(t,u;\,t',u';\,k)$, a split position $c\in[u,u+k]$ on $t$
induces the split position $c+u'-u$ on $t'$, and conversely.
Both directions include the endpoints.

In Algorithm~\ref{alg:interval-closure},
$\mathrm{Insert}(t,c,r)$ inserts a previously absent split position into
$C[t]$, adds $(t,c)$ to the worklist, and appends its justification $r$
to $R$. It has no effect if $c\in C[t]$. The justification records either
an initial split or an earlier split and the glue used to propagate it.

\begin{algorithm}[t]
\caption{Recover interfaces from interval constraints}
\label{alg:interval-closure}
\begin{algorithmic}[1]
\REQUIRE $\mathcal C=(\mathcal X,n,S,L,F)$
\STATE validate axis identifiers, positive integer widths,
       integer split positions and glue intervals
\STATE $C[t]\leftarrow\emptyset$ and $\mathrm{links}[t]\leftarrow[]$
       for every $t\in\mathcal X$
\STATE worklist $W\leftarrow\emptyset$; derivation record $R\leftarrow[]$
\FOR{each $t\in\mathcal X$}
  \STATE $\mathrm{Insert}(t,0,\mathrm{initial})$;
         $\mathrm{Insert}(t,n_t,\mathrm{initial})$
\ENDFOR
\FOR{each $(t,c)\in S$}
  \STATE $\mathrm{Insert}(t,c,\mathrm{initial})$
\ENDFOR
\FOR{each glue $\ell=(t,u;\,t',u';\,k)\in L$}
  \STATE insert its four endpoints as initial split positions
  \STATE append $(u,u+k,t',u'-u,\ell)$ to $\mathrm{links}[t]$
  \STATE append $(u',u'+k,t,u-u',\ell)$ to $\mathrm{links}[t']$
\ENDFOR
\WHILE{$W\ne\emptyset$}
  \STATE pop $(t,c)$ from $W$
  \FOR{each $(l,h,s,\delta,\ell)\in\mathrm{links}[t]$}
    \IF{$l\le c\le h$}
      \STATE $\mathrm{Insert}(s,c+\delta,(t,c,\ell))$
    \ENDIF
  \ENDFOR
\ENDWHILE
\STATE sort each $C[t]$ and form atoms $(t,c_i,c_{i+1})$
       between consecutive split positions
\STATE initialize one equivalence class per atom
\FOR{each glue $(t,u;\,t',u';\,k)\in L$}
  \FOR{each atom $(t,a,b)$ contained in $[u,u+k)$}
    \STATE merge its class with that of $(t',a+u'-u,b+u'-u)$
  \ENDFOR
\ENDFOR
\STATE let $\mathcal I$ be the resulting classes;
       assign each its common atom width
\STATE mark a class fixed if any of its atoms belongs to an axis in $F$
\STATE return $(C,\mathcal I,R)$
\end{algorithmic}
\end{algorithm}

\subsection{Bindings and head permutations}
\label{app:compiler-bindings}

For a linear layer, the input and output partitions divide the columns
and rows of its weight matrix into blocks. A block associated with
interfaces $I$ and $J$ transforms as
\[
  W_{JI}\mapsto P_J W_{JI}P_I^\top.
\]
Each bias segment is stored once and transforms by $P_J$.
Convolutional blocks retain the spatial kernel dimensions. Grouped
convolution and depthwise convolution use the corresponding block and
diagonal parameterizations. Normalization layers include their affine
parameters and, for BatchNorm, running means and variances. BatchNorm's
training counter is excluded.

\paragraph{Parameter sharing.}
The supported weight-tying rule covers an embedding layer and a linear
output layer without bias that share the same weight parameter and have
matching dimensions. The embedding uses fixed integer indices, with no
padding index or \texttt{max\_norm} update. The rule identifies the output
layer's input interface with the embedding's feature interface and its
output interface with the fixed vocabulary interface. These interfaces
must remain unsplit under the current binding rule. The shared parameter
is stored once, while both operations remain in the graph.
Unsupported parameter-sharing patterns are rejected.

\paragraph{Attention.}
For each supported attention layer, the constraints yield one query--key
interface and one value interface per head. A head permutation acts jointly
on the query, key and value projection blocks, their biases when present,
and the corresponding columns of the output projection in the
$[\mathrm{out},\mathrm{in}]$ convention. The output-projection bias is
unchanged. The checker verifies the head interfaces, operation occurrences
and all affected parameter bindings.

Supported head permutations exchange compatible interface factors and
therefore normalize $\Gam$. Together with the within-interface permutations,
they give $G=\Gam\rtimes E$ as defined in Section~\ref{sec:asserts}.
For $h$ heads of dimension $d$, the internal attention group is
$(S_d\times S_d)^h\rtimes S_h$.

\subsection{Guarantees}
\label{app:compiler-guarantees}

\begin{theorem}[Termination and exactness of interval closure]
\label{thm:compiler}
For every valid finite constraint system $\mathcal C$,
Algorithm~\ref{alg:interval-closure} inserts at most
$\sum_{t\in\mathcal X}(n_t+1)$ distinct split positions, including the
initial positions. It terminates with the least set containing the initial
positions and closed under propagation through every glue. This set and
the resulting interface partition are independent of processing order.

The permutation assignments satisfying the original constraints are
exactly those obtained by choosing independently one permutation in
$S_{n_I}$ for each non-fixed interface $I$, assigning the identity to
each fixed interface, and applying the chosen permutation to every atom
in the class. Their group is therefore isomorphic to
\[
  \Gam=\prod_{I\in\mathcal I^*}S_{n_I},
\]
where $\mathcal I^*$ denotes the non-fixed interfaces.
\end{theorem}

\begin{proof}
Each inserted position is an integer within its axis and is inserted at
most once. Every insertion is an initial position or follows by propagation
from an earlier insertion. It therefore belongs to every closed set
containing the initial positions. At termination, every applicable
propagation has been processed, so the result is the least such closed set.

Propagated splits are consequences of the original constraints.
A glue preserves its interval endpoints and equates the relative
permutations within the two intervals. Preserving a split in one interval
therefore implies preservation of the corresponding split in the other.
By induction, every admissible permutation assignment preserves all final
atoms. Closure matches atoms of equal width, and glues require their
permutations to agree. Thus each equivalence class carries one permutation,
which must be the identity if the class contains an atom of a fixed axis.

Conversely, assigning one permutation to each class preserves every split
and satisfies every glue. Fixed classes satisfy the fixed-axis constraints.
The resulting correspondence is an injective group homomorphism onto
the admissible permutation assignments.
\end{proof}

\paragraph{Function preservation.}
Assume that the operator and layout rules correctly describe the traced
computation, and that the parameter bindings and sharing constraints
implement their permitted transformations. Each operation is then
equivariant under the corresponding input, output and parameter
permutations. Composition, with external inputs and outputs fixed, gives
\[
  f_{g\theta}(x)=f_\theta(x),\qquad g\in\Gam,
\]
for every admissible input and model state $\theta$, including represented
buffers. The supported head permutations also preserve the function, so
the equality extends to $G$ by composition. The operator-specific arguments
for Theorem~\ref{thm:guarantee} are given in Appendix~\ref{app:soundness}.
These equalities are statements in real arithmetic; numerical tests allow
floating-point error.

Exactness is relative to $\mathcal C$. It does not imply that the compiler
recovers every function-preserving transformation or that distinct
permutations produce distinct parameter values. Continuous symmetries and
additional symmetries at particular parameter values are outside this claim.

\subsection{Verification, complexity and scope}
\label{app:compiler-checks}

\paragraph{Verification.}
The factor checker verifies each split derivation, inclusion of all initial
positions and closure under every glue. It independently computes the
connected components of the atom-matching graph and checks the reported
interface members, widths and fixed flags. It uses graph connectivity
rather than the solver's union--find implementation. The check establishes
correctness relative to the supplied constraints; correctness of the
operator rules requires the separate soundness arguments.

The head-permutation checker verifies the declared permutations against
the supported attention rules and parameter bindings. The parameter and
buffer coverage check compares entry counts. Equal counts alone do not
establish entrywise coverage; that property also depends on the packing
and parameter-sharing rules.

Separate numerical tests compare outputs before and after applying
interface permutations, head permutations and their compositions.
These tests are not part of Algorithm~\ref{alg:compile}.

\paragraph{Complexity.}
Let $B$ be the number of final split positions, $K$ the number of atoms,
$Q$ the number of matched atom pairs counted across glues, and $T$ the
number of split--link checks in the worklist loop. With standard disjoint-set
operations and expected constant-time dictionary access, the interval
solver takes
\[
  O(B+K+T+Q\alpha(K))
\]
apart from sorting, where $\alpha$ is the inverse Ackermann function.
Sorting takes at most $O(B\log B+K\log K)$. The factor checker takes
$O(B+K+Q)$ apart from the same sorting bound. These bounds exclude tracing,
parameter packing and verification of head permutations. The bound on
insertions does not bound the number of split--link checks.

\paragraph{Scope.}
The supported operations include linear layers, one- and two-dimensional
convolutions and batch normalization, layer normalization over one feature
dimension, selected elementwise operations and pooling, concatenation,
aligned addition and multiplication, embeddings and learned tensors.
Depthwise convolution with a channel multiplier greater than one uses
the grouped-convolution rule. Reshapes, transposes and reductions require
a supported coordinate-axis mapping.

Standard PyTorch self-attention is supported in the specified unmasked
configuration. Native GPT-2 has a separate implementation of the rule for
standard causal self-attention. Arbitrary attention masks and
cross-attention are unsupported.

A positional encoding supplied as an external input is fixed. Addition
therefore fixes the residual interfaces connected to its coordinate axis.
Supported learned positional embeddings are instead permuted jointly with
the residual features. The compiler does not inspect fixed tensor values
to recover additional permutations that preserve those values. Registered
positional buffers and other tensor attributes require a supported
transformation rule.

Models whose traces satisfy these rules require no additional
architecture-specific rules. Further details on supported operations and rejection cases are
given in Appendix~\ref{app:operators}.

\section{Encoder}

\subsection{The encoder in equations}
\label{app:encoder}

The encoder alternates between cell features of width $\dE$ and
interface-coordinate features of width $\dH$. Each layer applies
interface exchange, delivery, module updates and emission.
Local exchange follows module incidences; global attention uses
one summary per interface. Delivery broadcasts coordinate features
to cells, while emission pools cells back to coordinates.
Projections between $\dE$ and $\dH$ therefore act on coordinate
vectors rather than the full cell grid.

We first describe these updates, treating structural codes as
given. Appendix~\ref{app:encoder-codes} defines the codes and
numerical lift. The bindings determine which coordinates
communicate; the codes condition their features and messages.

\paragraph{Notation.}
We write $\norm$ for LayerNorm \citep{ba2016layer} without
affine parameters and $\mathrm{LN}$ for LayerNorm with learned
scale and shift. Both normalize only the feature dimension,
with $\epsilon=10^{-5}$; $\gelu$ denotes GELU
\citep{hendrycks2016gelu}.
Learned maps are affine unless specified otherwise, with biases
suppressed in the notation. Update maps are shared across objects
and architectures, with separate parameters in each layer.
Lift maps and structural embeddings are shared across layers.

\paragraph{Initialization.}
Let $w_{abs}$, $s=1,\ldots,k$, be the scalars folded into cell
$(a,b)$ of a rectangular unit $u$, with scalar codes
$c_s\in\R^{\dE}$. Its initial state is
\begin{equation}
\begin{aligned}
  \E_u^0[a,b]
  &=
  \frac1k\sum_{s=1}^{k}\phi(w_{abs},c_s)
  +M\,\norm(\tau_o)\\
  &\quad+\lambda_{u,0}(a)+\lambda_{u,1}(b),
\end{aligned}
\label{eq:fold}
\end{equation}
where $\phi$ is the scalar lift, $\tau_o\in\R^{\dH}$ is the
owning occurrence's code, and $M:\R^{\dH}\to\R^{\dE}$ is
bias-free. The terms $\lambda_{u,q}$ label coordinates on fixed
interfaces and vanish on permutable interfaces.
A diagonal unit uses one coordinate index and one label term.
The lift, codes and labels are defined in
Appendix~\ref{app:encoder-codes}.
The raw unit layout retains every parameter entry; its
fixed-width encoding need not be injective.

Interface states start at zero, or at
$\norm(C\,\mathrm{obs}_I)$ when coordinate-aligned observations
are supplied (Appendix~\ref{app:providers}).
Weightless states start at $u_o^0=\tau_o$.
When used, the $g$ global tokens
$\G\in\R^{g\times\dH}$ have learned initial values.
One emission using the first layer's maps seeds the interface
states before the layers begin.

\paragraph{Interface pass: local exchange.}
Each interface contributes an invariant summary
\[
  s_I=\mean_r\gelu\big(P\,\norm(\Hs_I[r])\big).
\]
Let $\mathcal P(o)$ be the binding of occurrence $o$, with
$I(p)$ denoting the interface of binding $p$.
Every ordered pair $p,q\in\mathcal P(o)$, including $p=q$,
contributes a message
\begin{equation}
  m^{(o)}_{p\to q}
  =L\,\gelu\!\left(
    \norm\!\left(s_{I(p)}+\kappa_{o,pq}\right)
  \right),
  \label{eq:local}
\end{equation}
where $\kappa_{o,pq}\in\R^{\dH}$ encodes the operation type
and sending and receiving roles.
This is message passing over module incidences
\citep{gilmer2017neural}.

Define the incoming binding pairs of interface $J$ by
\[
  \mathcal A_J=
  \{(o,p,q):o\in\Ops,\ p,q\in\mathcal P(o),\ I(q)=J\}.
\]
Their mean gives the local increment:
\begin{equation}
  \Delta_J^{\mathrm{loc}}
  =
  \frac{1}{\max(1,|\mathcal A_J|)}
  \sum_{(o,p,q)\in\mathcal A_J}m^{(o)}_{p\to q}.
  \label{eq:local-average}
\end{equation}
Binding remain distinct when they bind to the same interface,
so aggregation retains incidence multiplicities.
The increment is zero when no messages arrive.
For attention, local exchange connects the query--key and
value interfaces even though no parameter unit directly joins them.

\paragraph{Interface pass: global exchange.}
Global attention \citep{vaswani2017attention} acts on
$X=[s_1;\ldots;s_A;\G]$:
\begin{align}
  \Delta^{\mathrm{att}}
    &=\operatorname{MHA}(\norm X),
    \nonumber\\
  \Delta^{\mathrm{ff}}
    &=\operatorname{FFN}_{H}(X+\Delta^{\mathrm{att}}).
    \label{eq:global}
\end{align}
MHA includes learned query, key, value and output projections,
with no attention dropout or additional positional encoding.
The feed-forward block applies affine LayerNorm, a map
$\dH\to2\dH$, GELU and a map $2\dH\to\dH$.

Local and global increments update the interface coordinates
and global tokens:
\begin{align}
  \Hs_J[r]&\leftarrow
  \Hs_J[r]+\norm\!\left(
    \Delta_J^{\mathrm{loc}}
    +\Delta_J^{\mathrm{att}}
    +\Delta_J^{\mathrm{ff}}
  \right),
  \label{eq:broadcast}\\
  \G&\leftarrow
  \G+\Delta_{\G}^{\mathrm{att}}+\Delta_{\G}^{\mathrm{ff}}.
  \nonumber
\end{align}
Subscripts $J$ and $\G$ select the corresponding rows.
Each interface receives the same increment at every coordinate;
global tokens receive only the global increments.

\paragraph{Delivery: interfaces to modules.}
For unit axis $q$, bound to interface $I_q$ by slot map
$\sigma_q$ under role $\rho_q$, define
\[
  x^{(q)}
  =\Pi_{\downarrow}\,
   \gelu\!\left(
     \norm\!\left(
       \Hs_{I_q}[\sigma_q]+e_{\rho_q}^{\mathrm{recv}}
     \right)
   \right),
\]
where the receiving-role embedding is broadcast over gathered
coordinates and
$\Pi_{\downarrow}:\R^{\dH}\to\R^{\dE}$ is bias-free.
A rectangular unit receives
\begin{equation}
  \E_u[a,b]\leftarrow
  \E_u[a,b]+x^{(0)}[a]+x^{(1)}[b],
  \label{eq:deliver}
\end{equation}
and a diagonal unit receives its single-axis contribution.
Projection to cell width precedes broadcasting over the cell grid.

For a weightless binding $p$, let
\[
  h_p=\mean_r\Hs_{I(p)}[\sigma_p(r)].
\]
The weightless occurrence receives
\begin{equation}
  u_o\leftarrow u_o+
  \norm\!\left(
    \mean_{p\in\mathcal P(o)}
    \gelu\!\left(
      \norm\!\left(h_p+\widetilde e_p^{\mathrm{recv}}\right)
    \right)
  \right),
  \label{eq:deliver-weightless}
\end{equation}
where $\widetilde e_p^{\mathrm{recv}}$ additionally encodes
the binding position when ordered.

\paragraph{Module pass.}
A rectangular cell combines its current feature with summaries
of cells sharing either coordinate:
\begin{align}
  r_a&=A\,\norm\!\left(\mean_b\E_u[a,b]\right),
  \nonumber\\
  c_b&=A\,\norm\!\left(\mean_a\E_u[a,b]\right),
  \label{eq:cellcontext}\\
  \E_u[a,b]&\leftarrow
  \E_u[a,b]+
  W_{\mathrm{out}}\,
  \gelu\!\left(W_{\mathrm{in}}\E_u[a,b]+r_a+c_b\right).
  \label{eq:cellupdate}
\end{align}
The row and column summaries share the bias-free map
$A:\R^{\dE}\to\R^{\dE}$.
Both $W_{\mathrm{in}}$ and $W_{\mathrm{out}}$ have input
and output width $\dE$.
For a diagonal unit, the context is
$A\,\norm(\mean_a\E_u[a])$, broadcast to every coordinate.
The cell-normalized variant replaces the input to
$W_{\mathrm{in}}$ by $\norm(\E_u[a,b])$.

Weightless states use
\[
  u_o\leftarrow u_o+\operatorname{FFN}_{o}(u_o),
\]
where $\operatorname{FFN}_{o}$ applies affine LayerNorm,
an affine map, GELU and a second affine map, all at width $\dH$.

\paragraph{Emission: modules to interfaces.}
For rectangular unit axis $q$, let $v_q$ be the mean of
$\E_u$ over the other coordinate axis; for a diagonal unit,
let $v_q=\E_u$.
The outgoing messages are
\[
  e_{u,q}[a]
  =
  \gelu\!\left(
    \norm\!\left(
      \Pi_{\uparrow}v_q[a]+e_{\rho_q}^{\mathrm{send}}
    \right)
  \right),
\]
with shared bias-free
$\Pi_{\uparrow}:\R^{\dE}\to\R^{\dH}$.
A weightless binding contributes
$\gelu(\norm(u_o+\widetilde e_p^{\mathrm{send}}))$
at each bound coordinate.

Messages are scattered through their slot maps and averaged
at each receiving coordinate, giving $\bar e_I[r]$.
The update is
\begin{equation}
  \Hs_I[r]\leftarrow
  \Hs_I[r]+\norm\big(\bar e_I[r]+\bar t_I[r]\big),
  \label{eq:emit}
\end{equation}
where $\bar t_I[r]$ is the optional transbinding contribution
defined below. An empty incoming message set has mean zero.

\paragraph{Linear transport.}
Transport combines interface-coordinate states using raw
parameters as coefficients.
For a rectangular unit, define
$A_u[a,b]=k^{-1}\sum_s w_{abs}$.
If its axes bind to $I,J$ through $\sigma_0,\sigma_1$,
it sends
\begin{align}
  t_{u\to I}&=A_u\Hs_J[\sigma_1],
  \nonumber\\
  t_{u\to J}&=A_u^\top\Hs_I[\sigma_0].
  \label{eq:transport-linear}
\end{align}
A diagonal unit sends
$a_u[a]\Hs_I[\sigma(a)]$, where
$a_u[a]=k^{-1}\sum_s w_{as}$.
Messages are scattered through their receiving slot maps.
Their mean gives $\bar t_I[r]$, using a separate count
from ordinary emission.

All transport messages read the same pre-emission interface
states. Transport introduces no learned parameters;
when disabled, $\bar t_I[r]=0$.

\paragraph{Readouts.}
The loss-profile experiments give each interface equal weight:
\begin{align*}
  h&=\mean_I\mean_r\mathrm{LN}_{H}(\Hs_I[r]),\\
  \widehat y&=\operatorname{MLP}_{H}(h).
\end{align*}
The MLP has widths $\dH\to2\dH\to60$, with GELU between
its affine layers. This configuration has no global tokens.

Standard benchmarks and scaling experiments additionally
pool global tokens:
\begin{align*}
  z&=\mean_j\mathrm{LN}_{G}(\G[j]),\\
  \widehat y&=\operatorname{MLP}_{HG}([h;z]).
\end{align*}
The streams have separate affine normalizations.
The MLP has widths $2\dH\to2\dH\to C$, with GELU,
where $C=10$ for INR classification and $C=1$ for CNN
accuracy prediction and scaling measurements.

For parameter-valued outputs, the scalar with code $c_s$
in a cell is predicted by
\begin{equation}
  \operatorname{MLP}_{E}\!\left(
    \mathrm{LN}_{E}(\E_u[a,b])\odot g'(c_s)+b'(c_s)
  \right),
  \label{eq:perparam}
\end{equation}
where $g',b':\R^{\dE}\to\R^{\dE}$ are affine and
$\operatorname{MLP}_{E}$ has widths $\dE\to2\dE\to1$
with GELU. A cell containing $k$ scalars is queried
once per scalar code.
Inverting the unit reshapes and transpositions returns
predictions in the stored parameter-block layouts.

\paragraph{Batching and execution.}
Different architectures are batched as a disjoint union,
with exchange, global tokens and readout pooling separate
for each network. Matching layouts can instead be stacked
along a batch dimension.
Within a network, compatible units are stacked for cell
computations; large units are processed in tiles with
activation checkpointing \citep{chen2016training}.
These modes preserve the corresponding separate computations
up to floating-point error.

\subsection{Lift and structural codes}
\label{app:encoder-codes}

The lift combines scalar values with their parameter kinds
and positions within a cell. For a convolutional cell, these
positions are kernel offsets.
Code-dependent scaling associates values with their positions
before averaging, following FiLM \citep{perez2018film}.
Adding codes before a linear average would lose this association.
Mean aggregation accommodates different numbers of folded entries
\citep{zaheer2017}.

\paragraph{Parameter and operation codes.}
The scalar and occurrence codes are
\begin{align}
  c_s
  &=E_{\mathrm{param}}(\mathrm{key}_s)
    +G_{\mathrm{geo}}\gamma(\delta_s),
    \nonumber\\
  \tau_o
  &=E_{\mathrm{type}}(\mathrm{kind}(o))
    +D\,d_o
    +\Pi\,\widetilde\psi(\mathrm{depth}(o)).
    \label{eq:codes}
\end{align}
The parameter key identifies a kind, such as a linear weight
or normalization scale, rather than a source-model name.
The descriptor $d_o$ records operation attributes, and
$\mathrm{depth}(o)$ is the longest-path depth of the occurrence.
The maps $D$ and $\Pi$ are bias-free.

Convolutional offsets $\delta_s$ are measured from the kernel
centre. Folded learned tensors use centred indices in their
flattened layout, embedded as $(j,0)$; entries without an
offset use $\delta_s=0$.
The 32-dimensional map $\gamma$ contains sines and cosines of
\[
  \frac{\pi}{8}2^j\langle v,\delta\rangle,
  \qquad
  j=0,\ldots,3,\quad
  v\in\{(1,0),(0,1),(1,1),(1,-1)\}.
\]

\paragraph{Scalar features and lift.}
Let $a,b$ index the two coordinate axes of a rectangular unit
$u$, and let $w_{abs}$, $s=1,\ldots,k$, be the scalars folded
into cell $(a,b)$.
Each scalar has a code $c_s\in\R^{\dE}$ identifying its
parameter kind and position within the cell
(Equation~\eqref{eq:codes}).

We apply fixed multi-frequency Fourier features
\citep{tancik2020fourier} without rescaling the weights:
\[
  \psi(w)=
  [\,\sin(2^p w),\,\cos(2^p w)\,]_{p=-5}^{8}.
\]
The initial cell state is
\begin{equation}
\begin{aligned}
  \E_u^0[a,b]
  &=
  \frac1k\sum_{s=1}^{k}
  \left(
    V\psi(w_{abs})\odot\alpha(c_s)+\beta(c_s)
  \right)\\
  &\quad+M\,\norm(\tau_o)
  +\lambda_{u,0}(a)+\lambda_{u,1}(b).
\end{aligned}
\label{eq:fold}
\end{equation}
The affine maps $V$, $\alpha$ and $\beta$ have output width
$\dE$; $\alpha(c_s)$ and $\beta(c_s)$ provide the
code-dependent scale and shift.
The remaining terms add operation metadata and fixed-coordinate
labels, defined below.
For one scalar per cell, $k=1$. For a convolutional unit, each cell contains the $k=k_hk_w$
kernel coefficients for one input--output channel pair.

\paragraph{Binding roles and positions.}
Each role $\rho$ has separate sending and receiving embeddings
$e_\rho^{\mathrm{send}},e_\rho^{\mathrm{recv}}\in\R^{\dH}$.
The code used by local exchange is
\[
  \kappa_{o,pq}
  =
  E_{\mathrm{type}}(\mathrm{kind}(o))
  +e_{\rho(p)}^{\mathrm{send}}
  +e_{\rho(q)}^{\mathrm{recv}}.
\]
Thus local exchange uses the operation type and binding roles;
the full occurrence code $\tau_o$ initializes cell and
weightless states.

Unit-axis delivery and emission use the role embeddings directly.
For an ordered weightless binding $p$ at position $j_p$, they use
\[
  \widetilde e_p^{\mathrm{send/recv}}
  =
  e_{\rho(p)}^{\mathrm{send/recv}}
  +\Pi_p\widetilde\psi(j_p),
\]
where $\Pi_p$ is a shared bias-free map.
Unordered bindings use the role embedding alone.

\paragraph{Fixed-coordinate labels.}
For axis $q$ bound to interface $I_q$ by slot map $\sigma_q$,
\[
  \lambda_{u,q}(a)=
  \begin{cases}
    B_{\mathrm{fix}}\eta(\sigma_q(a)),
       & \text{if }I_q\text{ is fixed},\\
    0, & \text{otherwise},
  \end{cases}
\]
where $B_{\mathrm{fix}}$ is shared and bias-free.
The coordinate encoding is sinusoidal
\citep{vaswani2017attention}:
\[
  \eta(r)=
  [\,\sin(r\,10000^{-2j/32}),
     \cos(r\,10000^{-2j/32})\,]_{j=0}^{15}.
\]
Labels are added to cells on every fixed interface, including
input and output interfaces and residual interfaces fixed by
externally supplied positional encodings.
Learned positional embeddings transform through their parameter
bindings and do not by themselves fix the residual interface.
Coordinate labels are not added directly to interface states.

\section{Technical appendix}
\label{app:technical}

\subsection{Equivariance of the encoder}
\label{app:equivariance}

Every encoder stage respects the coordinate action of $\Gam$.
Shared pointwise maps and feature normalization commute with coordinate
permutations. Pooling removes the action on the reduced axis and preserves
it on the retained axes; gathering and scattering follow the bindings.
Linear transbinding contracts matching coordinate axes and therefore obeys
the same transformation law. Fixed-coordinate labels occur only where
the action is the identity. Optional observation features must transform
with their interface coordinates.

Supported module exchanges carry states between the corresponding
interfaces, modules and units. These exchanges preserve the codes,
layouts and typed bindings used by the encoder. Local aggregation
follows the exchanged incidences, while global attention follows the
permutation of interface summaries because it has no interface-position
encoding. Global-token outputs remain unchanged. Thus the encoder
respects the combined action $G=\Gam\rtimes E$.

The network-level readouts pool over coordinates and interfaces and
are therefore invariant. The per-parameter readout applies a shared
decoder to each cell and its scalar codes, so its outputs transform
with the represented parameters. These properties do not imply
that the encoder distinguishes all states in distinct $G$-orbits.

For Theorem~\ref{thm:guarantee}(b), the implemented exchanges move
complete attention heads: their query, key and value parameters,
including biases, move together with the corresponding output-projection
blocks. In evaluation mode, this merely reorders the sum of projected
head outputs and leaves the layer's output unchanged. The same argument
applies to the supported GPT-2 causal attention rule, whose mask is
shared across heads.
\subsection{Feature providers}
\label{app:providers}

A \emph{feature provider} supplies optional observations
$\mathrm{obs}_I\in\R^{n_I\times d_{\mathrm{ctx}}}$ on any subset of
interfaces, with one feature vector per coordinate. These initialize
the interface states before the seeding emission:
\begin{equation}
  \Hs_I^0=\norm(C\,\mathrm{obs}_I)
  \quad\text{where observations are supplied},\qquad
  \Hs_I^0=0\quad\text{otherwise},
  \label{eq:provider}
\end{equation}
where $C:\R^{d_{\mathrm{ctx}}}\to\R^{\dH}$ is shared across interfaces
and architectures. The remaining encoder is unchanged.

\paragraph{Learned probes.}
For the INR experiments, we use learned input probes
\citep{kofinas2024}. Evaluating each network at inputs
$z_1,\ldots,z_S$ gives
\begin{equation}
  \mathrm{obs}_I=
  [\,a_I(z_1),\ldots,a_I(z_S)\,]\in\R^{n_I\times S},
  \label{eq:probes}
\end{equation}
where $a_I(z)$ is the corresponding layer value, after its activation
where present. The probe inputs are shared across networks and
optimized jointly with the encoder, while the input networks'
weights remain fixed. Execution is differentiable with respect to
the probes, and the recorded values follow hidden-coordinate
permutations.

\subsection{Permutation actions in attention representations}
\label{app:attention-audit}

Consider an attention layer with two heads of width four and fixed
input and output coordinates. Independent query--key and value
permutations within each head, together with complete-head exchanges,
give the subgroup
\[
  G_{\mathrm{att}}
  =(\Sn{4}\times\Sn{4})^2\rtimes\Sn{2},
  \qquad |G_{\mathrm{att}}|=663{,}552.
\]
Its function preservation follows from Lemma~\ref{lem:attention}
and Theorem~\ref{thm:guarantee}(b).
Table~\ref{tab:orders} compares this action with those imposed by
the representations. Continuous symmetries and additional
parameter-specific symmetries are outside this comparison.

\begin{table}[t]
\centering
\small
\setlength{\tabcolsep}{5pt}
\begin{tabular}{@{}llrr@{}}
\toprule
Representation & Permutation action & Order & Non-symmetries \\
\midrule
Reference subgroup
  & $G_{\mathrm{att}}$ & 663,552 & 0 \\
QKV-triple graph \citep{kofinas2024}
  & $\Sn{8}$ & 40,320 & 39,168 \\
With fixed head labels \citep{lim2024}
  & $\Sn{4}\times\Sn{4}$ & 576 & 0 \\
Transformer-NFN \citep{tran2025tnfn}
  & $G_{\mathrm{att}}$ & 663,552 & 0 \\
UWE
  & $G_{\mathrm{att}}$ & 663,552 & 0 \\
\bottomrule
\end{tabular}
\caption{Internal permutation actions for two heads of width four,
with stream coordinates fixed. Non-symmetries are listed
transformations that do not preserve the attention function for
all parameter values. Transformer-NFN also supports continuous
basis changes, omitted from these counts.}
\label{tab:orders}
\end{table}

\paragraph{Representation-induced actions.}
The construction of \citet[Appendix~C.4]{kofinas2024} associates
each head coordinate with one node, combining its Q, K and V
weights into edge features. Without head labels, the eight nodes
can be relabelled arbitrarily, jointly permuting the rows of
$W_Q,W_K,W_V$ and the matching columns of $W_O$.
Here relabelling transports the numerical features; it need not
fix an individual numerically labelled graph.

The joint permutations that map each complete head block onto
a head block form $\Sn{4}\wr\Sn{2}$, of order
$24^2\cdot2=1{,}152$. The remaining $39{,}168$ mix coordinates
between heads and are not universal attention symmetries.
The joint QKV action also omits independent value permutations.
Adding fixed head labels, as proposed by \citet{lim2024},
restricts the action to $\Sn{4}\times\Sn{4}$ and additionally
excludes head exchange. UWE represents the two coordinate
factors separately and includes supported head exchanges,
giving $G_{\mathrm{att}}$.

\paragraph{Numerical check.}
We used stream width eight, independent standard normal projection
weights, zero biases and three standard normal input sequences
of six tokens, with random seed 20260922 and double-precision
evaluation. Among all $8!=40{,}320$ joint permutations, exactly
$1{,}152$ had maximum absolute output error below $10^{-9}$:
precisely those preserving the head partition.
All $4{,}000$ uniformly sampled elements of $G_{\mathrm{att}}$
also passed this threshold.

A separate encoder check used four rounds,
$\dE=8$, $\dH=32$, four attention heads, four global tokens
and linear transport, initialized with seed 20260923.
A complete-head exchange and a value-only permutation with
matching output-projection columns preserved its invariant
readout to within $10^{-16}$.
Exchanging one coordinate between heads, or permuting query
coordinates without their key coordinates, changed both the
attention output and this readout.
These are finite implementation checks; they do not establish
that the encoder distinguishes all distinct symmetry orbits.

\subsection{Representation size}
\label{app:counts}

Consider $L$ transformer encoder blocks of residual width $D$,
with $h$ heads of width $r=D/h$, feed-forward width
$D_{\mathrm{ff}}$, two LayerNorms per block and one final
LayerNorm. Embeddings and the task output layer are excluded.

\paragraph{Interfaces and coordinate states.}
Residual additions and normalization identify the residual
coordinate axes across the stack. Each head contributes one
query--key interface and one value interface, and each block
contributes one feed-forward hidden interface. Thus
\[
\begin{aligned}
  A &= 1+L(2h+1),\\
  S &= D+L(2hr+D_{\mathrm{ff}})
     =D+L(2D+D_{\mathrm{ff}}).
\end{aligned}
\]
These counts include fixed interfaces.
With biases and affine normalization parameters, the parameter
count is
\[
  P=L(4D^2+2DD_{\mathrm{ff}}+9D+D_{\mathrm{ff}})+2D.
\]
For $D=768$, $h=12$, $L=12$ and $D_{\mathrm{ff}}=3072$,
\[
  A=301,\qquad S=56{,}064,\qquad P=85{,}056{,}000.
\]
The cell count is close to $P$: matrix entries and biases are
stored individually, while normalization scale and shift share
cells.

\paragraph{Neuron graph.}
Applying the construction of
\citet[Appendices~C.3--C.4]{kofinas2024} gives $D$ input nodes.
Each block adds four sets of $D$ stream nodes, $hr=D$ head
nodes and $D_{\mathrm{ff}}$ feed-forward nodes.
The final LayerNorm adds another $D$ nodes, giving
\[
  N=(5L+2)D+LD_{\mathrm{ff}}=84{,}480.
\]
Full attention over these nodes has
$N^2=7{,}136{,}870{,}400$ ordered pairs.
Attention over UWE's interface summaries, with four global
tokens, has $(A+4)^2=305^2=93{,}025$.
These counts compare attention domains; a message-passing
graph network does not materialize all node pairs.

\subsection{Unsupported cases and validation}
\label{app:rejections}

Compilation rejects operations, layouts and state-sharing
patterns outside the rules of Appendix~\ref{app:operators}.
Examples include forward hooks, in-place operations and mutation
arguments, unsupported tensor attributes, undeclared aliases and
repeated stateful uses without a sharing rule.
The traced computation must return one tensor and admit a
consistent coordinate-axis assignment. Unsupported
data-dependent scalar extraction and coordinate-mixing layouts
are rejected.

Arbitrary attention masks and cross-attention remain unsupported;
the prescribed GPT-2 causal rule is a separate supported case.
There are no dedicated rules for rotary embeddings, recurrence
or mixture-of-experts routing. Operation-specific restrictions,
including those on embeddings, normalization and concatenation,
are listed in Appendix~\ref{app:operators}.

A final coverage check compares the number of packed state
entries with the expected number of parameters and buffers,
excluding BatchNorm's training counter. A mismatch is an error.
Equal counts alone do not establish one-to-one coverage; that
also depends on the packing and sharing rules.
The factor and head-exchange checks are described in
Appendix~\ref{app:compiler-checks}.

\subsection{Supported operations and state transformations}
\label{app:vocab}
\label{app:operators}
\label{app:rejections}

Table~\ref{tab:rules} specifies permutation constraints and
parameter and buffer bindings. Permutations act on resolved
coordinate axes; other indices are unchanged.
Interface partitions may divide a tensor into several units.
Rectangular units bind two coordinate axes, diagonal units bind
one, and remaining indices are folded within cells.
Shared state must receive the same transformation at every use.

\begin{table}[p]
\centering
\small
\setlength{\tabcolsep}{3.5pt}
\renewcommand{\arraystretch}{1.08}
\begin{tabular}{@{}p{0.22\linewidth}p{0.35\linewidth}p{0.37\linewidth}@{}}
\toprule
Operation & Permutation constraints & State bindings \\
\midrule

Linear layer; convolution with one group
& No input--output glue.
& Rectangular weights; biases bind the output.
  Kernel offsets are folded. \\

Depthwise convolution,
$g=n_{\mathrm{in}}=n_{\mathrm{out}}>1$
& Glue input and output.
& Diagonal kernels and biases; kernel offsets are folded. \\

Other grouped convolution, $g>1$
& Split input and output at group boundaries;
  no input--output glue.
& Rectangular kernels connect matching groups;
  biases bind the output. \\

BatchNorm; single-axis LayerNorm
& Glue input and output.
& Diagonal affine parameters.
  BatchNorm running statistics follow the channel action. \\

Pointwise activation, spatial pooling, evaluation-mode dropout
& Glue input and output.
& No stored state. \\

Aligned addition or multiplication
& Glue both operands and output.
& No stored state. \\

Coordinate-axis concatenation
& Glue each operand to its output interval;
  split at interval boundaries.
& No stored state. \\

Concatenation along another axis
& Glue all input coordinate axes to the output coordinate axis.
& No stored state. \\

Contiguous coordinate slice
& Glue the selected input interval to the output;
  split at its endpoints.
& No stored state. \\

Coordinate-preserving layout change or mean reduction
& Glue corresponding coordinate axes.
& No stored state. \\

Supported self-attention
& Split Q, K and V at head boundaries; glue Q to K.
  Declare complete-head exchanges separately.
& Rectangular projection blocks; Q and K share a head
  interface, while V shares its interface with the output
  projection's input. Biases bind projection outputs. \\

Embedding lookup
& Fix the vocabulary axis; constrain the feature axis
  through its uses.
& Rectangular vocabulary--feature units. \\

Learned floating-point tensor
& Constrain the feature axis through its uses.
& Diagonal units; remaining indices are folded. \\

Embedding tied to a linear output layer
& Glue the linear input to the embedding feature axis
  and its output to the fixed vocabulary axis.
& Store the shared weight once; retain both occurrences. \\

External tensor inputs and output
& Fix their coordinate axes.
& Input data are not packed into units. \\

\bottomrule
\end{tabular}
\caption{Implemented constraint patterns and state bindings.
Each row is subject to the restrictions below.
Fixed-axis and sharing constraints may further restrict the action.
Biases and shared state are stored without duplication.}
\label{tab:dictionary}
\label{tab:rules}
\end{table}

\paragraph{Binding roles.}
The vocabulary contains twelve roles: input (\texttt{in}),
output (\texttt{out}), normalization state (\texttt{gamma\_beta}),
depthwise channel (\texttt{channel}), shared query/key/value
projection input (\texttt{qkv\_in}), query, key and value projection
outputs (\texttt{q\_out}, \texttt{k\_out}, \texttt{v\_out}),
output-projection input and output (\texttt{o\_in},
\texttt{o\_out}), and main and skip operands
(\texttt{main}, \texttt{skip}) used in skip connections.

\paragraph{Operation families.}
Compilation uses evaluation-mode computation at the supplied
input shapes.
Convolutions, BatchNorm, max pooling, average pooling and adaptive
average pooling support one or two spatial dimensions.
LayerNorm normalizes one feature axis.
Activation rules cover ReLU, ReLU6, LeakyReLU, GELU, SiLU,
tanh, sigmoid, Hardswish, Hardsigmoid and functional
\texttt{hardtanh}; identity and evaluation-mode
\texttt{Dropout} and \texttt{Dropout2d} are also supported.
Depthwise convolution with channel multiplier greater than one
uses the grouped-convolution rule.

\paragraph{Layout restrictions.}
Reshape and flatten must preserve the coordinate axis without
merging it with another non-singleton dimension.
Expansion cannot change its width.
Selection uses literal indices and slices; the coordinate axis
must survive as a nonempty contiguous interval with unit stride.
Mean reductions require explicit non-coordinate axes.
Concatenation requires an explicit ordered operand list.
Addition and multiplication require two tensor operands with
aligned coordinate axes; scaled addition requires a nonzero
coefficient.
These are implementation restrictions, not maximal conditions
for equivariance.

\paragraph{Learned tensors and sharing.}
Embedding indices are fixed integer data; padding indices and
\texttt{max\_norm} updates are unsupported.
Learned tensor attributes must be floating-point parameters.
Other tensor attributes and buffers require an explicit
state-transformation rule.

The weight-tying rule covers one embedding and one bias-free
linear output layer with matching dimensions and the same
weight parameter. Their feature and vocabulary interfaces
must remain unsplit.
Additional aliases, repeated stateful uses and shared-storage
views are outside this rule.

\paragraph{Attention restrictions.}
The standard PyTorch rule requires batched self-attention:
Q, K and V receive the same tensor, and their projections use
a packed weight with equal embedding widths.
Masks, causal hints, added key/value bias tokens and
\texttt{add\_zero\_attn} are unsupported.

The native GPT-2 rule separately supports unpruned causal
self-attention with standard scaling and the prescribed
head-independent causal mask.
Cross-attention, padding masks, cached decoding, layer-dependent
attention scaling and custom attention implementations are
outside this rule.

\paragraph{Rejection and verification.}
The trace must return one tensor and admit a consistent
coordinate-axis assignment.
Forward hooks, in-place operations, mutation arguments,
unsupported data-dependent scalar extraction and undeclared
state sharing are rejected.
There are no dedicated rules for rotary embeddings, recurrence
or mixture-of-experts routing.

Constraint recovery is specified in
Appendix~\ref{app:compiler-constraints};
factor, head-exchange and state-coverage checks are described
in Appendix~\ref{app:compiler-checks}.

\subsection{Soundness of the operator rules}
\label{app:soundness}

For every supported operation, the emitted constraints and
state bindings imply
\begin{equation}
  f(P_{a_1}x_1,\ldots,P_{a_m}x_m;\theta')
  =P_b f(x_1,\ldots,x_m;\theta),
  \label{eq:contract}
\end{equation}
where $\theta'$ is the prescribed transformation of its
parameters and buffers.
The identity holds for every admissible input and state,
under the restrictions of Appendix~\ref{app:operators}.
The arguments below concern evaluation-mode computation in
real arithmetic. Batch, token and spatial indices are unchanged.

\paragraph{Linear layers and convolutions.}
For input and output permutations $P_a,P_b$,
\[
  (P_bWP_a^\top)(P_ax)+P_bb=P_b(Wx+b).
\]
Convolution satisfies this identity at each kernel offset;
summation over offsets preserves it.
For grouped convolution, the splits restrict permutations
to each group's input and output coordinates.
For depthwise convolution with channel multiplier one,
a common input--output permutation moves each channel's
kernel and bias together.

\paragraph{Coordinatewise operations and normalization.}
A shared scalar function or channelwise spatial pooling
commutes with coordinate permutations.
BatchNorm permutes affine parameters and any running
statistics with their channels.
Evaluation-mode dropout is the identity.

LayerNorm's mean and variance are permutation-invariant, so
\[
  \mathrm{LN}(Px;P\gamma,P\beta)
  =P\,\mathrm{LN}(x;\gamma,\beta).
\]
The identity also holds without affine parameters.

\paragraph{Addition, multiplication and concatenation.}
For aligned operands,
\[
  Px+\alpha Py=P(x+\alpha y),
  \qquad
  (Px)\odot(Py)=P(x\odot y).
\]
Coordinate-axis concatenation satisfies
\[
  \cat(P_1x_1,\ldots,P_mx_m)
  =(P_1\oplus\cdots\oplus P_m)\cat(x_1,\ldots,x_m).
\]
The splits preserve operand boundaries, and the glues equate
each operand's action with its output block.
Concatenation along another axis commutes with the common
coordinate permutation.

\paragraph{Slicing and layout operations.}
For a supported coordinate slice, let $R$ select its interval.
The endpoint splits preserve this interval, and the glue
equates its action with the output permutation. Hence
\[
  RP_{\mathrm{in}}=P_{\mathrm{out}}R.
\]
Selections on other axes leave the coordinate action unchanged.

Supported transposes, reshapes and flattenings relocate the
coordinate axis without mixing its coordinates.
Expansion and mean reduction affect only other axes.
These operations therefore satisfy Equation~\eqref{eq:contract}.
Fixed integer indices and shape computations have the identity
action; coordinate permutations do not change tensor shapes.

\paragraph{Embeddings, learned tensors and weight tying.}
Let $B\in\R^{V\times d}$ be an embedding matrix and $e_i$
the one-hot vector for a fixed vocabulary index.
A feature permutation $P$ gives $B'=BP^\top$, and
\[
  (B')^\top e_i=PB^\top e_i,
  \qquad
  B'(Ph)=Bh.
\]
Thus embedding lookup and the tied linear output layer induce
the same transformation of their shared weight.
A learned tensor source transforms along its bound feature
axis, supplying the required action to subsequent operations.

\begin{lemma}[Attention]
\label{lem:attention}
Consider a supported attention layer with row-wise features $X$,
projections
$Q_j=XW_{Q,j}^\top+\mathbf1b_{Q,j}^\top$,
and analogous $K_j,V_j$.
Its outputs are
\[
\begin{aligned}
  U_j
    &=\softmax\!\left(Q_jK_j^\top/\sqrt d+M\right)V_j,\\
  Y
    &=\cat(U_1,\ldots,U_h)W_O^\top+\mathbf1b_O^\top,
\end{aligned}
\]
where $d$ is the query--key width and $M$ is zero or the
prescribed GPT-2 causal mask.
Let $R_j,S_j$ independently permute head $j$'s query--key
and value coordinates.
Under $X'=XP_x^\top$ and the corresponding transformations
of projection weights and biases, with output permutation
$P_y$, the output satisfies $Y'=YP_y^\top$.
\end{lemma}

\begin{proof}
The affine rule gives
\[
  Q'_j=Q_jR_j^\top,\qquad
  K'_j=K_jR_j^\top,\qquad
  V'_j=V_jS_j^\top.
\]
Thus $Q'_j(K'_j)^\top=Q_jK_j^\top$.
The mask acts on token indices and is unchanged, so
$U'_j=U_jS_j^\top$.
Writing $U=\cat(U_1,\ldots,U_h)$ and
$S=\operatorname{diag}(S_1,\ldots,S_h)$, the output bindings give
$W'_O=P_yW_OS^\top$ and $b'_O=P_yb_O$.
Therefore
\[
  Y'
  =US^\top(P_yW_OS^\top)^\top+\mathbf1b_O^\top P_y^\top
  =YP_y^\top.
\]
\end{proof}

\begin{proof}[Proof of Theorem~\ref{thm:guarantee}]
For part (a), let $\gamma\in\Gam$ induce permutations $P_t$
on tracked tensor axes and a consistent transformed state
$\theta'$.
By Theorem~\ref{thm:compiler}, these permutations satisfy
the emitted constraints.

Proceed in topological order through the trace.
Fixed inputs and index data have the identity action;
learned tensor sources follow their state bindings.
If the inputs of operation $f$ obey their assigned actions,
Equation~\eqref{eq:contract} gives
\[
  t'
  =f(P_{t_1}t_1,\ldots,P_{t_m}t_m;\theta_f')
  =P_t f(t_1,\ldots,t_m;\theta_f)
  =P_t t.
\]
Sharing rules ensure that every $\theta_f'$ is induced by
the same $\theta'$.
Induction reaches the fixed output axis, establishing
function preservation.

For part (b), partition $W_O$ into head blocks $O_j$.
The attention output is
\[
  Y=\sum_{j=1}^{h}U_jO_j^\top+\mathbf1b_O^\top.
\]
A supported head exchange jointly permutes the Q, K and V
projections, their biases and the blocks $O_j$.
It reorders the summands and preserves the common mask,
leaving $Y$ unchanged.
The exchange rules enforce compatible bindings and shared-state
actions.
Composition establishes function preservation for
$G=\Gam\rtimes E$.
\end{proof}

\end{document}